\documentclass[journal,10pt]{IEEEtran}
\IEEEoverridecommandlockouts
\usepackage{cite}
\usepackage{bm,bbm,amsthm,amsmath,amssymb,amsfonts,mathrsfs}
\usepackage{algorithm}
\usepackage{algorithmic}
\usepackage{textcomp}
\usepackage{xcolor}
\usepackage{graphicx}  %插入图片的宏包
\usepackage{float}  %设置图片浮动位置的宏包
\usepackage{subfigure}  %插入多图时用子图显示的宏包
\usepackage{colortbl,booktabs,threeparttable}  %彩色表格需要加载的宏包
\usepackage{array}
\usepackage{booktabs}
\usepackage{boldline}
\usepackage{url}
\usepackage{color}
\usepackage{eqparbox}
\usepackage{caption}
\usepackage{stfloats}
\usepackage{verbatim}
\usepackage{cite}
\usepackage{multirow,booktabs,threeparttable}

\usepackage{siunitx}
\usepackage{balance}

\def\BibTeX{{\rm B\kern-.05em{\sc i\kern-.025em b}\kern-.08em
    T\kern-.1667em\lower.7ex\hbox{E}\kern-.125emX}}

\newtheorem{assumption}{Assumption} 
\newtheorem{theorem}{Theorem}

\newtheorem{lemma}{Lemma}
\newtheorem{corollary}{Corollary}

\newcommand{\defeq}{\stackrel{\text{\tiny \text{def}}}{=}}

\usepackage{subcaption}
\usepackage[bookmarks=false,colorlinks,citecolor=black, filecolor=black, linkcolor=black,urlcolor=black]{hyperref}
\allowdisplaybreaks[4]

\DeclareMathOperator*{\argsup}{arg\, sup}
\DeclareMathOperator*{\argmax}{arg\, max}
\DeclareMathOperator*{\argmin}{arg\, min}

\begin{document}
\title{Multi-Agent Probabilistic Ensembles with Trajectory Sampling for Connected Autonomous Vehicles}

\author{Ruoqi Wen, Jiahao Huang, Rongpeng Li, Guoru Ding, and Zhifeng Zhao
    \thanks{Copyright (c) 2024 IEEE. Personal use of this material is permitted. However, permission to use this material for any other purposes must be obtained from the IEEE by sending a request to pubs-permissions@ieee.org. (Corresponding Author: Rongpeng Li)}
    \thanks{R. Wen, J. Huang, and R. Li are with the College of Information Science and Electronic Engineering, Zhejiang University, Hangzhou 310058, China (email: \{wenruoqi, 3190103130, lirongpeng\}@zju.edu.cn).}
    \thanks{G. Ding is College of Communications and Engineering, Army Engineering University of PLA, Nanjing 210007, China (e-mail: dr.guoru.ding@ieee.org).}
    \thanks{Z. Zhao is with Zhejiang Lab, Hangzhou 311121, China, and also with the College of Information Science and Electronic Engineering, Zhejiang University, Hangzhou 310058, China (email: zhaozf@zhejianglab.com).}
    \thanks{Part of the paper has been accepted by IEEE Globecom 2023 (NativeAI Workshop) \cite{wen_multiagent_2023}.}
}
\maketitle

\begin{abstract}
%研究目的意义、研究内容、研究过程和最后的研究效果如何
Autonomous Vehicles (AVs) have attracted significant attention in recent years and Reinforcement Learning (RL) has shown remarkable performance in improving the autonomy of vehicles. In that regard, the widely adopted Model-Free RL (MFRL) promises to solve decision-making tasks in connected AVs (CAVs), contingent on the readiness of a significant amount of data samples for training. Nevertheless, it might be infeasible in practice, possibly leading to learning instability. In contrast, Model-Based RL (MBRL) manifests itself in sample-efficient learning, but the asymptotic performance of MBRL might lag behind the state-of-the-art MFRL algorithms. Furthermore, most studies for CAVs are limited to the decision-making of a single AV only, thus underscoring the performance due to the absence of communications. In this study, we try to address the decision-making problem of multiple CAVs with limited communications and propose a decentralized Multi-Agent Probabilistic Ensembles (PEs) with Trajectory Sampling (TS) algorithm namely \texttt{MA-PETS}. In particular, to better capture the uncertainty of the unknown environment, \texttt{MA-PETS} leverages PE neural networks to learn from communicated samples among neighboring CAVs. Afterward, \texttt{MA-PETS} capably develops TS-based model-predictive control for decision-making. On this basis, we derive the multi-agent group regret bound affected by the number of agents within the communication range and mathematically validate that incorporating effective information exchange among agents into the multi-agent learning scheme contributes to reducing the group regret bound in the worst case. Finally, we empirically demonstrate the superiority of \texttt{MA-PETS} in terms of the sample efficiency comparable to MFRL.
 
\end{abstract}

\begin{IEEEkeywords}
Autonomous vehicle control, multi-agent model-based reinforcement learning, probabilistic ensembles with trajectory sampling.
\end{IEEEkeywords}

\section{Introduction}
Recently, there has emerged significant research interest towards Connected Autonomous Vehicles (CAVs) with a particular emphasis on developing suitable Reinforcement Learning (RL)-driven controlling algorithms \cite{kiran2021deep, Hua2023MultiAgentRL} for the optimization of intelligent traffic flows \cite{trafficflows}, decision-making \cite{9793564}, and control of AVs \cite{gonzalez2015review}. Notably, these RL methods capably learn complex tasks for CAVs through effective interaction between agents and the environment, and existing RL works for CAVs can be classified as Model-Free Reinforcement Learning (MFRL) and Model-Based Reinforcement Learning (MBRL), the key differences of which lie in whether agents estimate an explicit environment model for the policy learning \cite{gronauer2022multi}.

Conventionally, MFRL, which relies on collected rewards on recorded state-action transition pairs, has been widely applied to model complex mixed urban traffic systems in multi-vehicle scenarios, showing excellent performance in various situations of autonomous driving \cite{MAPPO,zhang2022multi}. Typical examples of MFRL include MADDPG~\cite{MADDPG}, COMA~\cite{foerster2018counterfactual}, QMIX~\cite{rashid2020monotonic}, SVMIX\cite{SVMIX}. However, the computational complexity of most MFRL algorithms grows exponentially with the number of agents. To solve training data scarcity-induced out-of-distribution (OOD) problems, MFRL is typically required to repetitively interact with the real world to collect a sufficient amount of training data, which might be infeasible in practice and possibly lead to learning instability and huge overhead. On the contrary, due to the impressive sample efficiency, MBRL promises to solve CAVs issues\cite{9694460} more capably and starts to attain some research interest \cite{wu2022uncertainty,pan2021integrated}.  

Typically, MBRL is contingent on learning an accurate probabilistic dynamics model that can clearly distinguish between \textit{aleatoric} and \textit{epistemic} uncertainty~\cite{der2009aleatory}, where the former is inherent to the system noise, while the latter stems from sample scarcity and contributes to solving the OOD problem to a certain extent. Afterward, based on the learned dynamics model from the collected data, MBRL undergoes a planning and control phase by simulating model-consistent transitions and optimizing the policy accordingly. However, the asymptotic performance of MBRL algorithms has lagged behind state-of-the-art MFRL methods in common benchmark tasks, especially as the environmental complexity increases. In other words, although MBRL algorithms tend to learn faster, they often converge to poorer results~\cite{nagabandi2018neural, PETS}. Consequently, the deployment of model-based strategies to intricate tasks often encounters substantial compounded errors, which may impede trajectory learning. However, by extending single-agent RL \cite{MBRL_robotics} to multi-agent contexts through efficient communication protocols, one can compensate for the learning deficiency to some extent\cite{pmlr-v162-sessa22a, Tuynman2022TransferIR}.  In that regard, despite the simplicity of assuming the existence of a central controller, it might be practically infeasible or cost-effective to install such a controller in many real-world scenarios. Meanwhile, it is often challenging to establish fully connected communication between all agents, where the required communication overhead can scale exponentially\cite{MAPPO, MADDPG}. Instead, a more complex communication-limited decentralized multi-agent MBRL for CAVs becomes feasible, and specifying a suitable protocol for cooperation between agents turns crucial.

On the other hand, to theoretically characterize the sampling efficiency of RL, the concept of the regret bound emerges, which targets to theoretically measure the $T$-time-step difference between an agent's accumulated rewards and the total reward that an optimal policy (for that agent) would have achieved. Without loss of generality, for an $H$-episodic RL environment with $S$ states and $A$ actions, contingent on Hoeffding inequality and Bernstein inequality, tabular upper-confidence bound (UCB) algorithms can lead to a regret bound of $\tilde{O}(\sqrt{H^{4}SAT})$ and $\tilde{O}(\sqrt{H^{3}SAT})$ respectively \cite{NEURIPS2018_d3b1fb02,zhang2020almost}, where $\tilde{\mathcal{O}}(\cdot)$ hides the logarithmic factors. On the other hand, for any communicating Markov decision process (MDP) with the diameter $D$, Jaksch \textit{et al.} propose a classical confidence upper bound algorithm \texttt{UCRL2} algorithm (abbreviation of \textit{Upper Confidence bound for RL}) to achieve the regret bound $\tilde{O}(DS\sqrt{AT})$. In the literature, there only sheds little light on the regret bound of collaborative Multi-Agent Reinforcement Learning (MARL). In this paper, inspired by the regret bound of \texttt{UCRL2}, we investigate the regret bound of decentralized communication-limited MBRL, and demonstrate how communication among the multi-agents can be used to reduce the regret bound and boost the learning performance.
 
In this paper, targeted at addressing the sample efficiency issue in a communication-limited multi-agent scenario, we propose a fully decentralized Multi-Agent Probabilistic Ensembles with Trajectory Sampling (\texttt{MA-PETS}) algorithm. In particular, \texttt{MA-PETS} could effectively exchange collected samples from individual agents to neighboring vehicles within a certain range, and extend the widely adopted single-agent Probabilistic Ensemble (PE) technique to competently reduce both aleatoric and epistemic uncertainty while fitting the multi-agent environmental transition. Furthermore, \texttt{MA-PETS} employs Model Predictive Control (MPC) \cite{camacho2013model,rawlings2017model}, due to its excellence in scenarios that necessitate the joint prediction and optimization of future behavior, to generate appropriate control actions from a learned-model-based Trajectory Sampling (TS) approach. Compared to the existing literature, the contribution of the paper can be summarized as follows:

\begin{itemize}

\item We formulate the decentralized multi-agent decision-making issue for CAVs as a parallel time-homogeneous Markov Decision Process (MDP). On this basis, we devise a sample-efficient MBRL solution \texttt{MA-PETS} which utilizes a multi-agent PE to learn the unknown environmental transition dynamics model from data exchanged within a limited range, and calibrate TS-based MPC for model-based decision-making.

\item We analyze the group regret bound of \texttt{MA-PETS}, which is based on \texttt{UCRL2} and fundamentally different from the approach in \cite{Tuynman2022TransferIR} that utilizes a mere superposition of single-agent regret bounds. By incorporating the concept of clique cover \cite{CORNEIL1988109} for a limited-communication, undirected graph, we theoretically demonstrate that, even in the worst case, augmenting a multi-agent algorithm with additional communication information still results in a sub-linear group regret relative to the number of agents, and accelerates convergence. This validates that multiple agents jointly exploring the state-action space in similar environments could communicate to discover the optimal policy faster than individual agents operating in parallel.

\item We further illustrate our approach experimentally on a CAV simulation platform SMARTS~\cite{zhou2020smarts} and validate the superiority of our proposed algorithm over other MARL methods in terms of sample efficiency. Besides, we evaluate the impact of communication ranges on \texttt{MA-PETS} and demonstrate the contributing effects of information exchange.
\end{itemize}
%主要参数表格。
\begin{table}
    \caption{\centering{The Key Parameters for the Algorithms.}}%居中标题
    \begin{center}
    \label{tb: parameters}
            \begin{tabular}{|c|c|}
                \toprule	
                \multicolumn{1}{|c|}{\textbf{Notation}}& \multicolumn{1}{c|}{\textbf{Parameters Description}} \\
                \midrule
                $I$& Number of RL agents \\\hline
                $d$ & Communication range \\\hline
                $v_t$ & Speed per vehicle at time-step $t$ \\\hline
                $z_t$ & Position per vehicle at $t$  \\\hline
                $\bar{v}_{t}$ & Target velocity per vehicle at $t$  \\\hline
                $v_{t,a},v_{t,e}$ & Speed of vehicles \underline{a}head and b\underline{e}hind  \\\hline
                $l_{t,a},l_{t,e}$ & Distance of vehicles \underline{a}head and b\underline{e}hind \\\hline
                ${\Lambda}x_t$ & Travel distance per vehicle at $t$  \\\hline
                $K$ & No. of episodes  \\\hline
                $H$ & Length per episode  \\\hline
                $B$ & No. of Ensembles for dynamics Model  \\\hline
                $w$ & Horizon of MPC   \\\hline
                $Q$ & Candidate action sequences in CEM  \\\hline
                $X$ & Elite candidate action sequences in CEM \\\hline 
                $P$ & No. of particles  \\\hline
                $Y$ & Max iteration of CEM \\\hline
                $u_k(s,a)$ & Number of state-action $(s,a)$ visits in the episode $k$  \\\hline
                $N_k(s, a)$ & Number of state-action $(s,a)$ visits before episode $k$\\\hline
                $\mathbf{C}_{d}$ & Cliques cover of $\mathcal{G}_{d}$\\\hline
                $\bar{\chi}\left(\mathcal{G}_{d}\right)$ & Minimum number of cliques cover $\mathbf{C}_{d}$ \\
                \bottomrule
            \end{tabular}
    \end{center}
\end{table}

The remainder of the paper is organized as follows. Section \ref{sec: related works} briefly introduces the related works. In Section \ref{sec: preliminaries}, we introduce the preliminaries of MDPs and formulate the system model. In Section \ref{sec: MA-PETS} and Section \ref{sec: convergence proof}, we present the details of \texttt{MA-PETS} and unveil the effect of communication range on the convergence of the MARL via the group regret bound, respectively. Finally, Section \ref{sec: numerical experiments} demonstrates the effectiveness and superiority \texttt{MA-PETS} through extensive simulations. We conclude the paper in Section \ref{sec: conclusion}.

\section{Related Works}
\label{sec: related works}
\subsection{Multi-Agent Reinforcement Learning of CAVs}
Decision-making and planning are crucial components of CAV systems, with significant implications for enhancing the safety, driving experience, and efficiency of CAVs\cite{Hua2023MultiAgentRL,9564518}. The decision-making of CAVs in high-density mixed-traffic intersections belongs to one of the most challenging tasks and attracts significant research interest \cite{9762548,9939107,wevj15030099}, towards improving traffic efficiency and safety. In this regard, Centralized Training and Decentralized Execution (CTDE) has become one of the most popular paradigms. For instance, within the actor-critic MFRL framework, \cite{foerster2018counterfactual, MADDPG} proposes using a central critic for centralized training coupled with multiple actors for distributed execution. However, as the number of agents increases or the action space expands, the computational complexity of MFRL algorithms like \texttt{MA-DDPG} can rise exponentially, presenting a significant challenge. To mitigate this computational burden, \cite{rashid2020monotonic} proposes \texttt{QMIX} to involve the value decomposition of the joint value function into separate individual value functions for each agent. In addition, \cite{SVMIX} proposes to adopt a stochastic graph neural network to capture dynamic topological features of time-varying graphs while decomposing the value function. To avoid the large number of communication overheads that frequent information exchange incurs, \cite{FIRL} develops a consensus-based optimization scheme on top of the periodic averaging method, which introduces the consensus algorithm into Federated Learning (FL) for the exchange of a model's local gradients.

Although these studies showcase the robust post-training performance of RL in real-time applications, the MFRL algorithm is still plagued by high overhead, stemming from the extensive computational and sampling requirements. This issue is particularly pronounced in scenarios such as CAVs\cite{drl_AV}, where data acquisition is challenging, interactions between agents and the environment are inefficient, and each iteration yields only a limited amount of data. Therefore, as data samples become scarce, the performance of MFRL can quickly degrade and become unstable.

\subsection{Model-Based Reinforcement Learning} 
To address the sampling efficiency and communication overhead issues in MFRL, MBRL naturally emerges as an alternative solution. Unfortunately, MBRL suffers from performance deficiency, as it might fail to accurately estimate the uncertainty in the environment and characterize the dynamics model, which belongs to a critical research component in MBRL. For example, \cite{huseljic2021separation} proposes a DNN-based method that is, to some extent, qualified to separate aleatoric and epistemic uncertainty while maintaining appropriate generalization capabilities, while PILCO~\cite{deisenroth2011pilco} marginalizes aleatoric and epistemic uncertainty of a learned dynamics model to optimize the expected performance. \cite{PETS} takes advantage of MPC to optimize the RL agent’s behavior policy by predicting and planning within the modeled virtual environment at each training step. Another type of MBRL falls within the scope of Dyna-style methods \cite{MB-DQN}, where additional data is generated from interactions between the RL and the virtual environment, thus improving the efficiency of the RL. \cite{wu2022uncertainty} proposes an uncertainty-aware MBRL and verifies that it has competitive performance as the state-of-the-art MFRL. \cite{MAPPO} proposes an RL training algorithm \texttt{MAPPO} that incorporates a prior model into \texttt{PPO} algorithm to speed up the learning process using current centralized coordination methods. However, centralized coordination methods face issues such as high resource-intensiveness, inflexibility, and annoying delays, since the central node must process a large amount of information and handle decision-making for the entire system. 
In addition, there is still little light shed on the multi-agent MBRL scenario, especially in the communication-limited case \cite{survey_mbrl}. Considering the substantial communication overhead associated with centralized coordination methods, we mainly focus on the fully decentralized multi-agent MBRL algorithm under communication constraints.   

\subsection{Regret bounds of Reinforcement Learning}
Understanding the regret bound of online single-agent RL-based approaches to a time-homogeneous MDP has received considerable research interest. For example, \cite{kearns_nearoptimal_2002} discusses the performance guarantees of a learned policy with polynomial scaling in the size of the state and action spaces. \cite{UCRL} introduces a \texttt{UCRL} algorithm and shows that its expected online regret in unichain MDPs is $O(\log T)$ after $T$ steps. Furthermore,  
\cite{UCRL2} proposes the a \texttt{UCRL2} algorithm % designed for un-discounted communicating MDPs. 
which is capable of identifying an optimal policy through Extended Value Iteration (EVI), by conjecturing a set of plausible MDPs formed within the confidence intervals dictated by the Hoeffding inequality. Moreover, \cite{UCRL2} demonstrates that the total regret for an optimal policy can be effectively bounded by $\tilde{O}(DS\sqrt{AT})$. Afterward, many variants of \texttt{UCRL2} have been proposed for the generation of tighter regret bounds. For instance, \cite{ortner_online_2012} proposes a \texttt{UCCRL} algorithm that derives sub-linear regret bounds $\tilde{O}\left(T^{\frac{2+\alpha}{2+2 \alpha}}\right)$ with a parameter $\alpha$ for un-discounted RL in continuous state space. By using more efficient posterior sampling for episodic RL, \cite{osband_more_2013} achieves the expected regret bound $\tilde{O}(\iota S\sqrt{AT})$ with an episode length $\iota$. \cite{hao_bootstrapping_2019} introduces a non-parametric tailored multiplier bootstrap method, which significantly reduces regret in a variety of linear multi-armed bandit challenges. Similar to the single-agent setting, agents in MARL attempt to maximize their cumulative reward by estimating value functions, and the regret bound can be analyzed as well. For example, \cite{9867146} proposes that a specific class of online, episodic, tabular multi-agent Q-learning problems with UCB-Hoeffding exploration through communication yields a regret $\mathcal{O}\left(\sqrt{IH^4SAT\iota} \right)$, where $I$ is the number of RL agents and $\iota:=$ $\log (SATI/p)$. Nevertheless, despite the progress on single-agent regret bound and the group regret bound for fully connected multi-agent cases \cite{Tuynman2022TransferIR}, the group regret of MBRL in communication-limited multi-agent scenarios remains an unexplored area of research.  Consequently, we integrate dynamic graph theory, akin to that is found in \cite{9867146}, with an analysis of small blocks to explore the group regret bounds within a distinct MARL algorithm. % 文献【20】需要在这里补充一两句话
 
\section{Preliminaries and System Model}
\label{sec: preliminaries}

In this section, we briefly introduce some fundamentals and necessary assumptions of the underlying MDPs and the framework of MBRL. On this basis, towards the decision-making issue for CAVs, we highlight how to formulate MBRL-based problems.

\subsection{Preliminaries}
\subsubsection{Parallel Markov decision process}
The decision-making problem of CAVs can be formulated as a collection of \textit{parallel time-homogeneous stochastic MDPs} $\mathcal{M}:= \left\{{\rm MDP}\left(\mathcal{S}^{(i)}, \mathcal{A}^{(i)}, \mathcal{P}^{(i)}, \mathcal{R}^{(i)}, H\right)\right\}_{i=1}^I $ among agents $i \in [I]:=\{1,2,\cdots, I\}$ for $H$-length episodes \cite{bernstein2002complexity}. Despite its restrictiveness, parallel MDPs provide a valuable baseline for generalizing to more complex environments, such as heterogeneous MDPs. Notably, agents have access to identical state and action space (i.e., $\mathcal{S}^{(i)}=\mathcal{S}^{(j)}$ and $\mathcal{A}^{(i)}=\mathcal{A}^{(j)}$, $\forall i, j \in[I]$). We assume bounded rewards, where all rewards are contained in the interval $\left[0, r_{\max}\right]$ with mean $\bar{r}(s, a)$. Besides, we assume the transition functions $\mathcal{P}^{(i)}$ and reward functions $\mathcal{R}^{(i)}$ only depend on the current state and chosen action of agent $i$, and conditional independence applies between the transition function and the reward function. Furthermore, we focus exclusively on stationary policies, denoted as $\pi^{(i)}: \mathcal{S}^{(i)} \rightarrow \mathcal{A}^{(i)}$, which indicates the taken action $a_t^{(i)}\in\mathcal{A}^{(i)}$ for an agent $ i\in[I]$ after observing a state $s_t^{(i)} \in \mathcal{S}^{(i)}$ at time-step $t$. Based on the taken action, the agent $i$ receives a reward $r_t^{(i)} = \mathcal{R}^{(i)}\left(s_t^{(i)},a_t^{(i)}\right)$, and the environment transitions to the next state $s_{t+1}^{(i)}$ according to an unknown dynamics function $\mathcal{P}^{(i)}\left(s_{t+1}^{(i)}\mid s_t^{(i)},a_t^{(i)}\right): \mathcal{S}^{(i)} \times \mathcal{A}^{(i)} \rightarrow \mathcal{S}^{(i)}$. In other words, each agent interacts with the corresponding homogeneous MDP and calculates an expectation over trajectories where action $a_{t+1}^{(i)}$ follows the distribution $\pi^{(i)}(s_t^{(i)})$. 

On the other hand, an MDP is called \textit{communicating}, if for any two states $s,s^{\prime} \in \mathcal{S}$, there always exists a policy $\pi$ that guarantees a finite number of steps $L_{\pi}(s,s^{\prime})$ to transition from state $s$ to state $s^{\prime}$ when implementing policy $\pi$. In such a communicating MDP, the opportunity for recovery remains viable even if an incorrect action is executed. As previously mentioned, the diameter $D$ of the communicating MDP $M^{(i)}$, which measures the maximal distance between any two states in the communicating MDP, can be defined as
\begin{align}
    D(M^{(i)}):=\max _{s, s^{\prime} \in \mathcal{S}} \min _{\pi^{(i)}} L_{\pi^{(i)}}\left(s, s^{\prime}\right).
\end{align}
\noindent Moreover, unlike cumulative rewards over $T$ steps, we take account of average rewards, which can be optimized through a stationary policy $\pi^{\ast(i)}$ \cite{puterman1994markov}. The objective is for each agent to learn a policy $\pi^{\ast(i)}$ that maximizes its individual average reward after any $T$ steps, i.e,
\begin{align}
    \pi^{\ast(i)}=\argsup_{\pi \in \Pi^{(i)}} \left\{\liminf _{T \rightarrow+\infty} \mathbb{E}_{\pi}\left[\frac{1}{T} \sum\nolimits_{t=1}^T r_t^{(i)}\right]\right\}, %\mathcal{R}^{(i)}\left(s_t^{(i)}, a_t^{(i)}\right)
\end{align}
where $\Pi^{(i)}$ is the set of the plausible stationary randomized policies. 

\begin{lemma}[Lemma 10 of \cite{gajane_variational_2019}]
\label{lem:optimal_T_bound}
Consider $M^{(i)}$ as a time-homogeneous and communicating MDP with a diameter of $D$. Let $\mathcal{R}^{\ast(i)}_T(M^{(i)})$ represent the optimal $T$-step reward, and $\rho^{\ast(i)}$ denote the optimal average reward under the reward function $\mathcal{R}^{(i)}$. It can be asserted that for any MDP $M^{(i)}$, the disparity between the optimal $T$-step reward and the optimal average reward is minor, capped at a maximum of order $D$. Mathematically,
\begin{align}
    \mathcal{R}^{\ast(i)}_T(M^{(i)}) \leq T\rho^{\ast(i)}(M^{(i)})+D\cdot r_{\max}^{(i)}.
\end{align}
\end{lemma}
Hence, by Lemma \ref{lem:optimal_T_bound}, the optimal average reward $\rho^{\ast(i)}$ serves as an effective approximation for the expected optimal reward over $T$-steps. To evaluate the convergence of the RL algorithms for each agent $i$, we consider its regret after an arbitrary number of steps $T$, defined as
\begin{align}
    \operatorname{Regret}^{(i)}(T):=T\rho^{\ast(i)}-\sum\nolimits_{t=1}^T r_t^{(i)}.
\end{align}
%where $r_t^{(i)}$ is the reward that the agent $i$ actually obtains at step $t$. 
In the CAV scenario, as it is challenging to consider each agent's regret bound independently, we introduce the concept of multi-agent group regret, which is defined as 
\begin{align}
\label{eq:group_regret}
    \operatorname{Regret}_G(T)=\sum\nolimits_{i=1}^I \operatorname{Regret}^{(i)}(T).
\end{align}

\subsubsection{Model-Based Reinforcement Learning}
An MBRL framework typically involves two phases (i.e., dynamics model learning, and planning \& control). In the dynamics model learning phase, each agent $i \in [I]$ estimates the dynamics model $\tilde{\mathcal{P}}^{(i)}$ from collected environmental transition samples by a continuous model (e.g., DNNs). Afterward, based on the approximated dynamics model $\tilde{\mathcal{P}}^{(i)}$, the agent simulates the environment and makes predictions for subsequent action selection. In that regard, we can evaluate $w$-length action sequences $\mathbf{A}_{t:t+w-1}^{(i)}=\left(a_t^{(i)}, \ldots, a_{t+w-1}^{(i)}\right)$ by computing the expected reward over possible state-action trajectories $\tau_{t:t+w-1}^{(i)}=\left({s}_t^{(i)}, {a}_t^{(i)}, \cdots, {s}_{t+w-1}^{(i)}, {a}_{t+w-1}^{(i)}, {s}_{t+w}^{(i)}\right)$, and optimize the policy accordingly. Mathematically, this planning \& control phase can be formulated as
\begin{align}
\label{eq: expected reward}
    \mathbf{A^\ast}_{t:t+w-1}^{(i)}&=\argmax\nolimits_{\mathbf{A}_{t:t+w-1}^{(i)}} \sum\nolimits_{t^{\prime}=t}^{t+w-1} r_{t^{\prime}}^{(i)}\\
    & s.t.\quad s_{t+1}^{(i)} \sim \tilde{\mathcal{P}}^{(i)}\left(s_t^{(i)}, a_t^{(i)}\right).\nonumber %\mathcal{R}^{(i)}\left(s_{t^{\prime}}^{(i)}, a_{t^{\prime}}^{(i)}\right)
\end{align}

Limited by the non-linearity of the dynamics model, it is usually difficult to calculate the exact optimal solution of \eqref{eq: expected reward}. However, many methods exist to obtain an approximate solution to the finite-level control problem and competently complete the desired task. Common methods include the traversal method, Monte Carlo Tree Search (MCTS)~\cite{MCTS}, Iterative Linear Quadratic Regulator (ILQR)~\cite{ILQR}, etc. 

\subsection{System Model}
\label{sec:systemModel}
\begin{figure}
    \centering
    \includegraphics[width=0.95\linewidth]{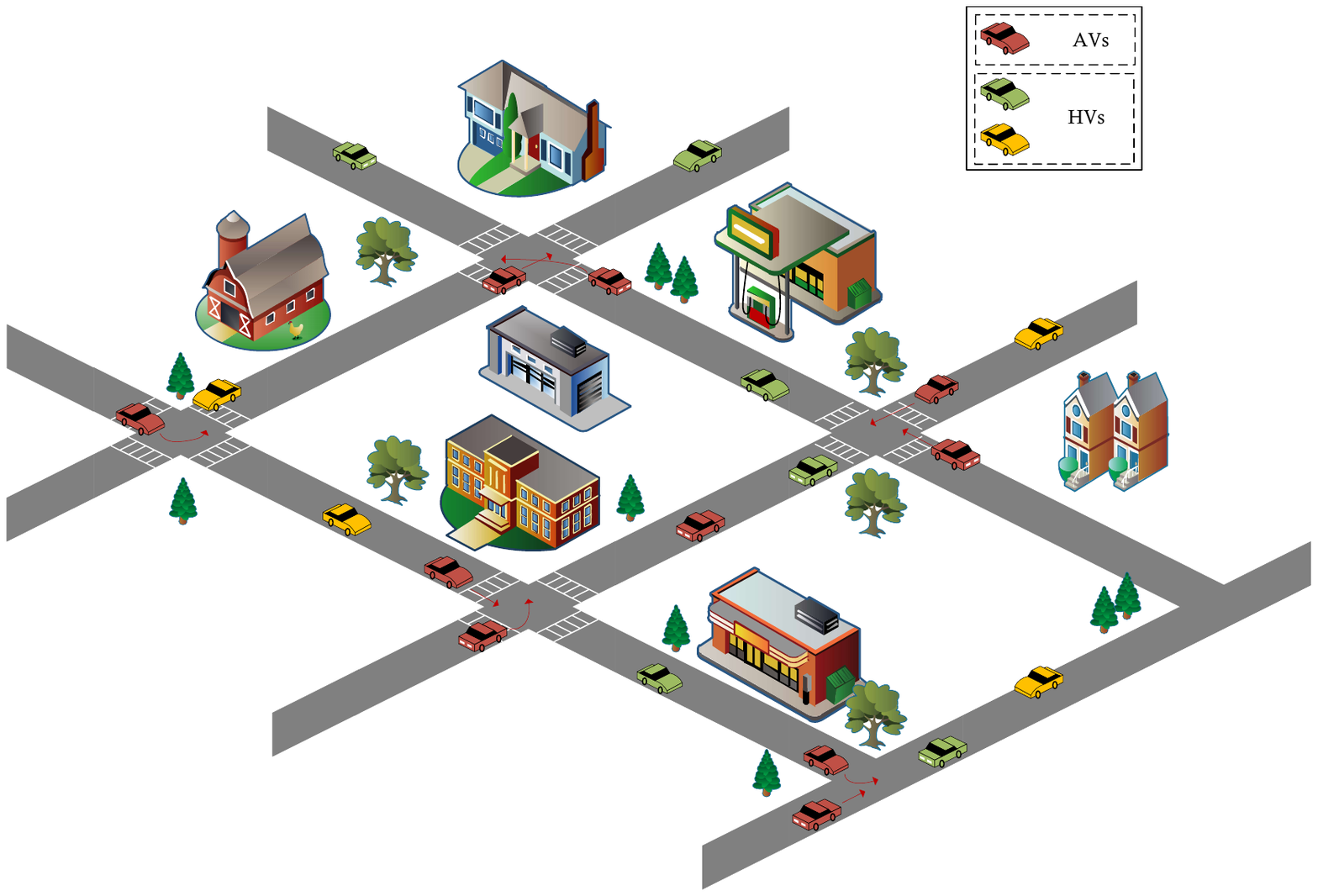}
    %\vspace{-1cm}
    \caption{The complex mixed urban traffic control scenario with unsignalized lane merging and intersections.}
    %\vspace{-0.4cm}
    \label{fig1}
\end{figure}
As illustrated in Fig. \ref{fig1}, we consider a mixed autonomy traffic system model with $I$ CAVs and some human-driving vehicles (HVs). Consistent with the terminology of parallel MDPs, the state $s_{t}^{(i)} \in \mathcal{S}^{(i)} = \left(v_t^{(i)},x_t^{(i)}, y_t^{(i)},v_{t,a}^{(i)},v_{t,e}^{(i)},l_{t,a}^{(i)},l_{t,e}^{(i)} \right)$ for agent $i\in [I]$ at time-step $t$ encompasses information like its velocity $v_t^{(i)} \in \mathbb{R}$, its position $z_t^{(i)}=\left(x_t^{(i)}, y_t^{(i)}\right) \in \mathbb{R}^2$, the velocity and relative distance of the vehicles \underline{a}head and b\underline{e}hind (i.e., $v_{t,a}^{(i)},v_{t,e}^{(i)},l_{t,a}^{(i)},l_{t,e}^{(i)} \in \mathbb{R}$).
Meanwhile, each vehicle $i$ is controlled by an adjustable target velocity $\bar{v}_{t}^{(i)} \in \mathbb{R}$ (i.e., $a_{t}^{(i)}=\left(\bar{v}_{t}^{(i)}\right)$). Furthermore, we assume the dynamics model shall be learned via interactions with the environment. At the same time, the reward function can be calibrated in terms of the velocity and collision-induced penalty, thus being known beforehand.

% And for probabilistic dynamics, we represent the transition function $f^{i}$ as some parameterized distribution family: ${f_\theta^i}\left(s_{t+1}^{(i)}\mid s_t^{(i)},a_t^{(i)}\right)=\mathcal{P}^i\left(s_{t+1}^{(i)}\mid s_t^{(i)},a_t^{(i)}:\theta\right)$, where $\theta$ the parameter vector of the neural network. The environment transition function $f^{i}$ is unknown and can only be learned via interactions with the environment. For ease of notation, we assume that vehicles’ reward functions $\mathcal{R}^i$ are known and are often designed depending on the tasks’ goals. Our proposed algorithm can be easily extended to unknown reward functions using standard techniques. In our scenario, the goal of vehicles is to maintain maximum velocity while avoiding a collision. Besides, safety is the most important criterion for a vehicle. We give the vehicle a greater penalty when a collision happens in order to hope the vehicle learns to drive safely to avoid collisions. 
\begin{figure*}[!ht]
    \centering
    \includegraphics[width=.95\linewidth]{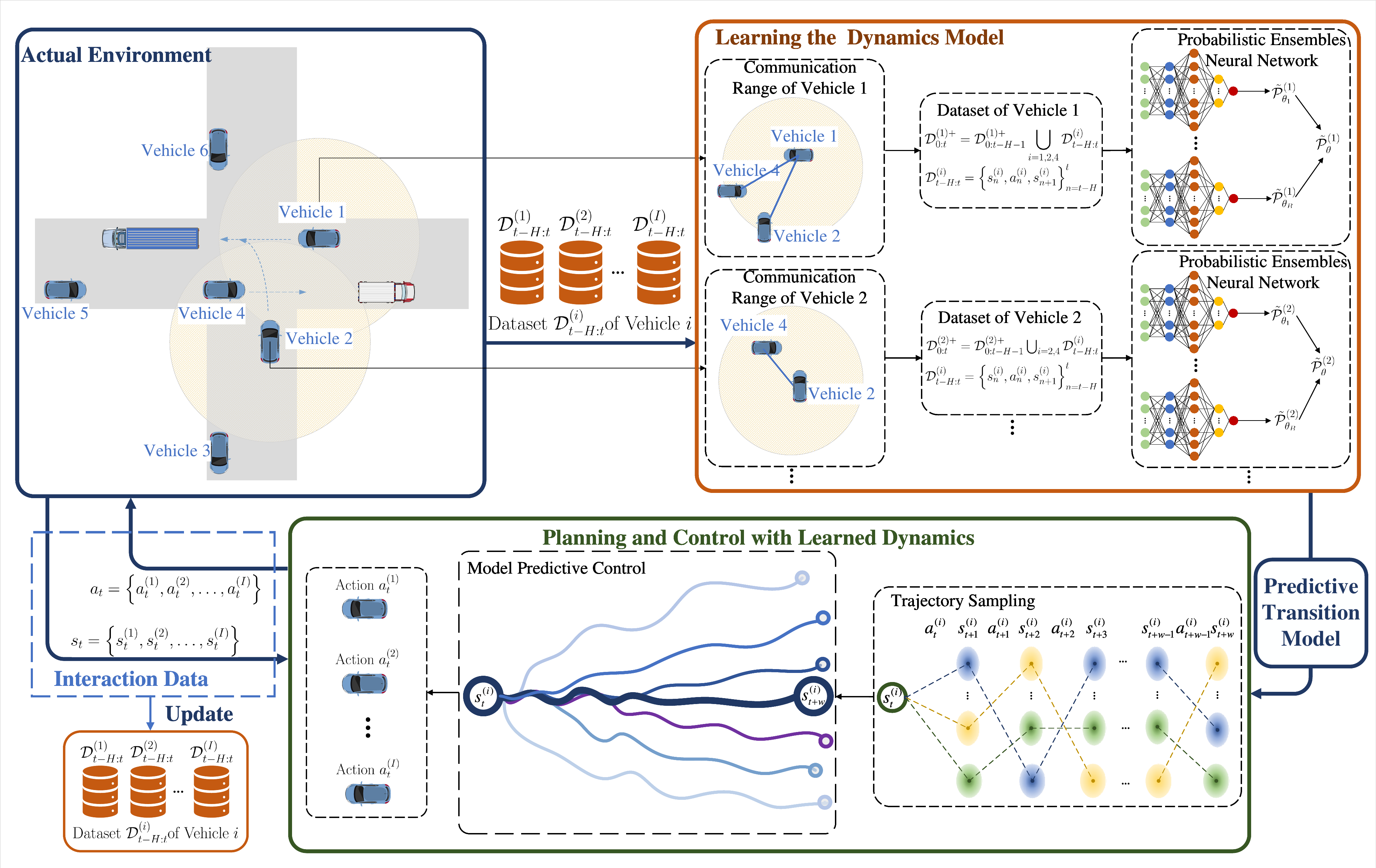}
    \caption{The illustration of the \texttt{MA-PETS} algorithm for CAVs.}
    \label{fig2}
\end{figure*}

For simplicity, let $t_k$ be the start time of episode $k$. Within the framework of decentralized multi-agent MBRL, we assume the availability of state transition dataset $\mathcal{D}_{0:t}^{(i)}=\left\{s_{t^{\prime}}^{(i)}, a_{t^{\prime}}^{(i)}, s_{t^{\prime}+1}^{(i)}\right\}_{t^{\prime}=0}^t$ and each agent approximates $\tilde{\mathcal{P}}^{(i)}, \forall i\in[I]$ by DNNs parameterized by $\theta$ based on historical dataset and exchanged samples from neighbors. In particular, for a time-varying undirected graph $\mathcal{G}_k$ constituted by $I$ CAVs, each CAV can exchange its latest $H$-length\footnote{Notably, the applied length could be episode-dependent but is limited by the maximum value $H$, as will be discussed in Section \ref{sec: MA-PETS}.} dataset $\mathcal{D}_{t_k-H:t_k}^{(i)}$ with neighboring CAVs within the communication range $d\in [0, D(\mathcal{G}_k)]$ at the end of an episode $k$, where $D(\mathcal{G}_k)$ is the dynamic diameter\footnote{We slightly abuse the notation of the diameter for an MDP and that for a time-varying undirected graph.} of graph $\mathcal{G}_k$. Therefore, when $d=0$, the multi-agent MBRL problem reverts to $I$ parallel single-agent cases. Mathematically, the dataset for dynamics modeling at the end of an episode $k$ could be denoted as $\mathcal{D}_{0:t_k}^{(i)+}:= \mathcal{D}_{0:t_k-H-1}^{(i)+} \cup \{ \mathcal{D}_{t_k-H:t_k}^{(j)}\}$ for all CAV $j \in \mathcal{F}_{d,k}^{(i)}$, where the set $\mathcal{F}_{d,k}^{(i)}$ encompasses all CAVs satisfying $\text{dis}_t(i,j) \leq d$ with $\text{dis}_t(i,j)$ computing the Euclidean distance between CAV $i$ and $j$ at the last time-step $t_k+H$. Furthermore, the planning and control objective for multi-agent MBRL in a communication-limited scenario can be re-written as
\begin{align}
\label{eq:model-based}
     &J\left(\tilde{\mathcal{P}}^{(i)}, d\right)=\mathbb{E}_{\boldsymbol{\pi} \sim \tilde{\mathcal{P}}^{(i)}}\left[\sum\nolimits_{t^{\prime}=0}^{H} r_{t^{\prime}}^{(i)}\mid s_0^{(i)}\right]\nonumber\\
    &s.t.\quad\tilde{\mathcal{P}}^{(i)}\propto \mathcal{D}_{0:t}^{(i)+}, 0\leq\textit{d}\leq D(\mathcal{G}_k) \\
    &\quad\quad s_0^{(i)} \sim p(s_0), s_{t+1}^{(i)} \sim \tilde{\mathcal{P}}^{(i)}(s_t^{(i)},a_t^{(i)}),\forall t\in[H], \nonumber % \mathcal{R}^{(i)}\left(s_{t^{\prime}}^{(i)}, a_{t^{\prime}}^{(i)}\right)
\end{align}
where  $\tilde{\mathcal{P}}^{(i)}$ is the learned dynamics model based on $\mathcal{D}_{0:t}^{(i)+}$. In this paper, we resort to a model-based PETS solution for solving \eqref{eq:model-based} and calculating the group regret bound.

\section{The \texttt{MA-PETS} Algorithm}
\label{sec: MA-PETS}
In this section, we discuss how to extend PETS~\cite{PETS} to a multi-agent case and present the dynamics model learning and planning \& control phases in \texttt{MA-PETS}, which can be depicted as in Fig.~\ref{fig2}.

\subsection{Learning the Dynamics Model}
In \texttt{MA-PETS}, we leverage an ensemble of bootstrapped probabilistic DNNs to reduce both aleatoric and epistemic uncertainty. In particular, in order to combat the aleatoric uncertainty, we approximate the dynamics model at time-step $t$ (i.e., $\tilde{\mathcal{P}}\left(s_{t+1}^{(i)}\mid s_t^{(i)},a_t^{(i)}\right),\forall i \in [I]$) by a probabilistic DNN, by assuming that the conditioned probability distribution of $s_{t+1}$ follows a Gaussian distribution $\mathcal{N}(\mu_\theta, \sigma_\theta)$ with mean $\mu_\theta$ and diagonal covariance $\sigma_\theta$ parameterized by $\theta$. In other words,
\begin{align}
    \tilde{\mathcal{P}}_{\theta}^{(i)}=\argmin_{\theta} \frac{1}{\left\vert \hat{\mathcal{D}}^{(i)+}\right\vert}\operatorname{loss}_{\mathrm{Gauss}}(\theta),
\end{align}
where $\hat{\mathcal{D}}^{(i)+}$ denotes a sub-set sampled from $\mathcal{D}^{(i)+}$ and $\operatorname{loss}_{\mathrm{Gauss}}(\theta)$ is defined as in \eqref{eq:loss_gauss}. Consistent with the definition of states and actions in Section \ref{sec:systemModel}, $\mu_\theta\left(s, a\right) \in \mathbb{R}^7$ and $\sigma_\theta\left(s, a\right) \in \mathbb{R}^{7 \times 7}$.
\begin{figure*}[!h]
    \begin{align}
    \operatorname{loss}_{\mathrm{Gauss}}(\theta)
    & =-\sum\nolimits_{\left(s_t, a_t, s_{t+1}\right) \in \hat{\mathcal{D}}^{(i)+} } \log \tilde{\mathcal{P}}^{(i)}_{\theta}\left(s_{t+1}\mid s_t, a_t\right) \label{eq:loss_gauss}\\
    & =\sum\nolimits_{\left(s_t, a_t, s_{t+1}\right) \in \hat{\mathcal{D}}^{(i)+} } \left\{ \left[\mu_{\theta}\left(s_t, a_t\right)-s_{t+1}\right]^{\top} \sigma_{\theta}^{-1}\left(s_t, a_t\right) \left[ \mu_{\theta}\left(s_t, a_t\right)-s_{t+1}\right]+\log \operatorname{det} \sigma_{\theta}\left(s_t, a_t\right)\right\}\nonumber
    \end{align}
    \hrulefill
\end{figure*}    
    % \begin{align}
    % \tilde{\mathcal{P}}\left(s_{t+1}\mid s_t,a_t\right)
    % & =\mathcal{N}\left(\mu_\theta\left(s_t, a_t\right), \sigma_\theta\left(s_t, a_t\right)\right) \\
    % & \sim \operatorname{Pr}\left(s_{t+1} \mid s_t,a_t\right)\nonumber
    % \end{align}

In addition, in order to mitigate the epistemic uncertainty, which arises primarily from the lack of sufficient data, a PE method is further adopted. Specifically, \texttt{MA-PETS} consists of $B$ bootstrap models in the ensemble, each of which is an independent and identically distributed probabilistic DNN $\tilde{\mathcal{P}}_{\theta_b}^{(i)},b \in [B]$ from a unique dataset $\hat{\mathcal{D}}_b^{(i)+}$ uniformly sampled from $\mathcal{D}^{(i)+}$ with the same size. 
\begin{comment}
\cite{lakshminarayanan2017simple} treats the ensemble model as a uniformly weighted mixed model, that is, the prediction is a mixture of Gaussian distributions. Consequently, the mean and variance of the ensemble model\footnote{Notably, though the ensemble model $\tilde{\mathcal{P}}_{\theta}^i$ is directly applicable, \cite{PETS} shows that particle-based TS method on a random selection of bootstrap models, which will be introduced in the next sub-section, could achieve superior results.} $\tilde{\mathcal{P}}_{\theta}^{(i)}$ can be further constructed from bootstrap models as
\begin{align}
\label{eq:ensembles_mean}
    \mu_{\theta}^{(i)}\left(s_t, a_t\right)
    &=\frac{1}{B} \sum\nolimits_{b=1}^B \mu_{\theta_b}^{(i)} %
\end{align}
and 
\begin{align}
\label{eq:ensembles_variance}
    \sigma_{\theta}^{(i)}\left(s_t, a_t\right)=\sqrt{\frac{1}{B}\sum\nolimits_{b=1}^B\left[(\sigma_{\theta_b}^i)^2+ (\mu^i_{\theta_b})^2\right]-(\mu_{\theta}^i)^2}
\end{align}
respectively. 
\end{comment}
Typically, \cite{PETS} points out that $B=5$ yields satisfactory results. 
% $ \left(\mu_{\theta_b}^{(i)}(s_t, a_t), \sigma_{\theta_b}^{(i)} (s_t, a_t) \right) $

\subsection{Planning and Control with Learned Dynamics}
Based on a learned dynamics model $\tilde{\mathcal{P}}^{(i)}$, \texttt{MA-PETS} tries to obtain a solution to \eqref{eq: expected reward} by resorting to a sample-efficiency controller MPC~\cite{camacho2013model}. Generally speaking, MPC excels in providing high-precision trajectory planning, thus reducing the risk of collisions during vehicle operation. Moreover, it can predict both aleatoric and epistemic uncertainty accurately within the learned dynamics model \cite{wevj15030099}. Without loss of generality, assume that for a time-step $t$, agent $i$ observes the state $s_t^{(i)}$. Afterward, agent $i$ leverages the Cross-Entropy Method (CEM) \cite{botev2013cross} to generate $Q$ candidate action sequences within a $w$-horizon of MPC. Initially, each candidate action sequence $\mathbf{A}_{t:t+w-1}^{(q,i)}, \forall q \in [Q]$ is sampled following a Gaussian distribution $\mathcal{N}(\mu_\phi,\sigma_\phi )$ parameterized by $\phi$. Meanwhile, agent $i$ complements possible next-time-step states from the dynamics model ensembles, by adopting particle-based TS \cite{girard2002multiple}. In other words, for each candidate action sequence $\mathbf{A}_{t:t+w-1}^{(q, i)}$, TS predicts plausible state-action trajectories by simultaneously creating $P$ particles that propagate a set of Monte Carlo samples from the state $s_{t}^{(i)}$. Due to the randomness in the learned and assumed time-invariant dynamics model, each particle $p \in [P]$ can be propagated by $s_{t'+1,p}^{(i)}={\tilde{\mathcal{P}}_{\theta_b}^{(i)}}\left(s_{t',p}^{(i)}, where a_{t'}^{(i)}\right)$ ($t' \in [t, t+w-2]$), and $\tilde{\mathcal{P}}_{\theta_b}^{(i)}, b\in [B]$ is a randomly selected dynamics model in the ensemble model. Therefore, the set of plausible state-action trajectories for each candidate action sequence $\mathbf{A}_{t:t+w-1}^{(q, i)}=\left(a_t^{(q, i)}, \ldots, a_{t+w-1}^{(q, i)}\right)$ ($q\in[Q]$) consists of $P$ parallel propagated states with the same action sequence $\mathbf{A}_{t:t+w-1}^{(q, i)}$. Afterward, given the calibrated known reward function $\mathcal{R}^{(i)}$, the evaluation of a candidate action sequence can be derived from the average cumulative reward of $P$ different parallel action-state trajectories, that is, 
\begin{align}
\label{eq: evaluate action sequances}
J(\mathbf{A}_{t:t+w-1}^{(q,i)})=\mathbb{E}\left[
    \frac{1}{P} \sum_{p=1}^P \sum_{t^{\prime}=t}^{t+w-1} \mathcal{R}^{(i)}\left(s_{p,t^\prime}^{(i)}, a_{t^\prime}^{(q,i)}\right)\right].
\end{align}
After sorting $Q$ candidate action sequences in terms of the evaluation $J(\mathbf{A}_{t:t+w-1}^{(i)})$, elite candidate sequences $X \leq Q$  can be selected to update the Gaussian-distributed sequential decision-making function $\mathcal{N}(\mu_\phi,\sigma_\phi )$. Such a TS-based CEM procedure can repeat until convergence of $\phi$. Therefore, it can yield the desired optimal action sequence satisfying \eqref{eq: expected reward} and the MPC controller executes only the first action $a_t^{\ast(i)}$, transitions the actual environment into the state $s_{t+1}^{(i)}$, and re-calculates the optimal action sequence at the next time-step.
Finally, we summarize our model-based MARL method \texttt{MA-PETS} in Algorithm \ref{al: MA_PETS}. 
\begin{algorithm}[t]
    \caption{The \texttt{MA-PETS} algorithms}
	\label{al: MA_PETS}
	\hspace*{\algorithmicindent} 
        \textbf{Input}: communication range $d$, initial state $\mathcal{N}(\mu_{\phi_0},\sigma_{\phi_0} )$, rarity parameters $\alpha$, max iteration of CEM $Y$, accuracy $\epsilon$.
	\begin{algorithmic}[1]
		\STATE Initialize vehicle’ dataset $\mathcal{D}^{(i)},\mathcal{D}^{(i)+}=\emptyset, \forall i \in[I]$;
        \STATE Randomly sampling action in the initial episode and update $\mathcal{D}_{0:H}^{(i)+},\forall i \in [I]$;
		\FOR{Vehicle $i$ and episode $k$, $\forall i \in [I], \forall k \in [K]$}
            \STATE Train the virtual environment dynamics ensemble model $\tilde{\mathcal{P}}^{i}$ given $\mathcal{D}^{(i)+}$ for each vehicle $i$;
            \FOR{step $t=1,2, \ldots H$ do}
                \WHILE{ $j\leq Y$ and $\sigma_{\phi_j}-\sigma_{\phi_{j-1}} \leq \epsilon $ }
                \STATE Each vehicle samples $Q$ candidate action sequences $\mathbf{A}_{t:t+H-1}^{(1,i)}, \cdots,\mathbf{A}_{t:t+H-1}^{(Q,i)}\sim_{iid}\mathcal{N}(\mu_{\phi_j},\sigma_{\phi_j})$;
                \STATE Propagate state particles $s_{t,p}^{(q,i)}$ by using ensemble model and TS method;
                \STATE Evaluate action sequences by \eqref{eq: evaluate action sequances};
                \STATE Select the top $\alpha\%$ of action sequences with the high rewards;
		      \STATE  Update $\mathcal{N}(\mu_{\phi_{j+1}},\sigma_{\phi_{j+1}})$ distribution by high reward action sequences;
		        \ENDWHILE
                \STATE Only execute first action $a_{t}^{\ast(i)},\forall i \in [I]$ from optimal actions $\mathbf{A^\ast}_{t:t+w-1}^{(i)}$;
                \STATE Obtain reward $r_{t}^{(i)},$ and observe next-step state $s_{t+1}^{(i)}$;
                \ENDFOR
                \STATE Determine neighboring vehicle $i^{\prime},\forall \text{dis}(i,i^{\prime}) \leq d$ (i.e., within the communication range $d$ of vehicle $i$) to exchange information;%信息传递
		      \STATE Update dataset $\mathcal{D}^{(i)+} \leftarrow {\mathcal{D}^{(i)+} \cup \{ \mathcal{D}_{t-H:t}^{(i^{\prime})}\}}$.
                \ENDFOR             
	\end{algorithmic}
\end{algorithm}

\section{Regret Bound of \texttt{MA-PETS}}
\label{sec: convergence proof}
Next, we investigate its group regret bound of \texttt{MA-PETS}, by first introducing how to construct an optimistic MDP from a set of plausible MDPs through EVI. Based on that, we derive the group regret bound facilitated by communications within a range of $d$. 
\subsection{EVI for an Optimistic MDP}
Beforehand, in order to better exemplify the sample efficiency, we make the following assumptions.
\begin{assumption}
    \label{assump: quantization}
    The continuous state and action space (i.e., $\mathcal{S}^{(i)}$ and $\mathcal{A}^{(i)}$) can be quantized. Correspondingly, $S=|\mathcal{S}^{(i)}|$ and $A=|\mathcal{A}^{(i)}|$ denote the size of the discrete state space and action space, respectively. 
\end{assumption}

\begin{assumption}
    \label{assump:episode_length}
    All agents in \texttt{MA-PETS} could enter into the next episode in a ``simultaneous'' manner, while the episodic length can be tailored to meet some pre-defined conditions.
\end{assumption}

\begin{assumption}
    \label{assump:rmax1}
  For simplicity of representation, we can scale the reward $r$ from $[0,r_{\max}]$ to $[0,1]$.
\end{assumption}

\noindent Notably, Assumption \ref{assump: quantization} holds naturally, if we neglect the possible discretization error of DNN in \texttt{MA-PETS}, as it has a trivial impact on understanding the sample efficiency incurred by information exchange. The quantization discussed here involves discretizing continuous state and action spaces. Common discretization methods \cite{sutton2018reinforcement} include uniform quantization, which divides each dimension of the continuous state space into intervals of equal width (bins). These intervals define the precision of discretization, with each interval representing a discrete state. Besides, non-uniform quantization may discretize based on the distribution or importance of states, and clustering algorithms (such as k-means or DBSCAN \cite{Ratliff-2009-10259},\cite{azizzadenesheli2018efficient}) can group points in the state space, with each cluster center representing a discrete state. Meanwhile, text encoding (e.g., Tile Coding \cite{Barto2003RecentAI}), which involves overlaying different tiles on the states and assigning one bit to each tile, belongs to an alternative solution. Therefore, if an agent is in a certain position, the corresponding discrete state can be represented by one bit-vector, with activated positions set to $1$ and all other positions set to $0$. On the other hand, Assumption \ref{assump:episode_length} could be easily met by intentionally ignoring the experienced visits of some ``diligent" agents. Moreover, Assumption \ref{assump:rmax1} implies unanimously scaling the group regret bound by $r_{\max}$, which does not affect learning the contributing impact from inter-agent communications.

Based on these assumptions, we incorporate the concept of a classical MBRL algorithm \texttt{UCRL2} into the analyses of \texttt{MA-PETS}. % For example, both of them estimate the reward and transition probability functions from existing historical data, and then iteratively compute the choice of action based on the currently learned model. Therefore, it becomes of vital importance to devise a multi-agent variant of \texttt{UCRL2} in multi-agent scenarios with a limited range of information. 
% In the parallel MDP construction, we assume that there is no coupling between reward and transition probability. 
\iffalse
\begin{assumption}
    \begin{itemize}
        \item \texttt{MA-PETS} utilizes a continuous action and state space, while \texttt{MA-UCRL2} operates on a discrete action and state space.%量化
        \item  Both systems, \texttt{MA-PETS} and \texttt{MA-UCRL2}, feature dynamic episode lengths, but under different conditions: \texttt{MA-PETS} triggers a new episode upon a collision or when reaching the episode's upper limit. In contrast, \texttt{MA-UCRL2} initiates a new model learning when the visits of part of state-action pairs doubles. The core principle for both is to ensure the validity of the learned model. In \texttt{MA-PETS}, a collision is deemed a sign of the model's inadequacy in mitigating major collision risks. Consequently, both systems initiate model training only after accumulating sufficient new data, ensuring the model's stability and reducing major collision risks.
    \end{itemize}
\end{assumption}
\fi
In particular, \texttt{UCRL2} primarily implements the ``optimism in the face of uncertainty'', and performs an EVI through episodes $k=1,2,\dots$, each of which consists of multiple time steps. The term $N_k^{(i)}(s, a)$ denotes the number of visits to the state-action pair $(s, a)$ by agent $i$ before episode $k$. Next, \texttt{UCRL2} determines the optimal policy for an optimistic MDP, choosing from a collection of plausible MDPs $\mathcal{M}^i$ constructed based on the agents' estimates and their respective confidence intervals, as commonly governed by the Hoeffding inequality for an agent \cite{hoeffding1994probability}.

Consistent with \texttt{UCRL2}, we allow each agent $i$ to enter into a new episode $k + 1$ once there exists a state-action pair $(s, a)$ that has just been played and satisfies $u_k^{(i)}(s, a)=N^{(i)+}_k(s, a)$. Here $u_k^{(i)}(s,a)$ denotes the number of state-action $(s,a)$ visits by agent $i$ in the episode $k$ and $N^{(i)+}_k(s,a)= \max\left\{1, N^{(i)}_k(s,a)\right\}$. By Assumption \ref{assump:episode_length}, all agents could enter into the next episode in a ``simultaneous'' manner, by intentionally ignoring the experienced visits of a ``diligent" agent $i$ after meeting $u_k^{(i)}(s, a)=N^{(i)+}_k(s, a)$ in episode $k$. Therefore, at the very beginning of episode $k+1$, for each agent $i$, $N_{k+1}^{(i)}(s,a) = N_{k}^{(i)}(s,a) + u_k^{(i)}(s,a)$ for all state-action pairs$(s, a)$ . This setting is similar to the doubling criterion in single-agent \texttt{UCRL2} and ensures that each episode is long enough to allow sufficient learning. Furthermore, different from the single-agent \texttt{UCRL2}, at the end of an episode $k$, consistent with in \texttt{MA-PETS}, each agent $i$ has access to all the set of neighboring CAVs $\mathcal{F}_{d,k}^{(i)}$, so as to obtain the latest dataset $\mathcal{D}_{t_k-H:t_k}^{(j)}$, $\forall j \in \mathcal{F}_{d,k}^{(i)}$ and constitute $\mathcal{D}_{0:t_k}^{(i)+}$. Again, according to Assumption \ref{assump:episode_length}, $H = \sum_{s,a}u_k^{(i)}(s,a)$. Afterward, the EVI proceeds to estimate the transition probability function $\mathcal{P}^{(i)}(\cdot \mid s, a)$. In particular, various high-probability confidence intervals of the true MDP $M^{(i)}$ can be conjectured according to different concentration inequalities. For example, based on Hoeffding’s inequality \cite{hoeffding1994probability}, the confidence interval for the estimated transition probabilities can be given as
 %只计算转移概率置信区间
\begin{align}
 % &\mid \tilde{r}^{(i)}(s, a)-\hat{r}^{(i)}_k(s, a) \mid \leq \sqrt{\frac{7 \log ({2SANt_k^{(i)}}/{\delta})}{2 N_t^{(i)+}(s, a)}}, \\
&\left\|\tilde{\mathcal{P}}^{(i)}(\cdot \mid s, a)-\hat{\mathcal{P}}^{(i)}_k(\cdot \mid s, a)\right\|_1 \leq \sqrt{\frac{14S \log ({2At_k}/{\delta})}{N_k^+(s, a)}},
\label{eq: confidence interval}
\end{align} 
and it bounds the gap between the EVI-conjectured (estimated) transition matrix $\tilde{\mathcal{P}}^{(i)}(\cdot\mid s, a)$ derived from the computed policy $\tilde{\pi}^{(i)}_k$ and the empirical average one $\hat{\mathcal{P}}^{(i)}_k(\cdot\mid s, a)$ in episode $k$. Besides, $\delta \in \left[0,1\right]$ is a pre-defined constant.

Furthermore, a set $\mathcal{M}_k^{(i)}$ of plausible MDPs \eqref{eq:plausible_mdp} is computed for each agent $i$ as
\begin{align}
\label{eq:plausible_mdp} 
\mathcal{M}^{(i)}_k \defeq & \bigg\{\tilde{M}^{(i)} = \langle \mathcal{S}^{(i)}, \mathcal{A}^{(i)}, \mathcal{R}^{(i)}, \tilde{\mathcal{P}}^{(i)}\rangle\bigg\}.
\end{align}
\noindent Correspondingly, a policy for an optimistic MDP can be attained from learning in the conjectured plausible MDPs, and then the impact of the communication range $d$ on the performance of \texttt{MA-PETS} can be analyzed, in terms of the multi-agent group regret in \eqref{eq:group_regret}.

\subsection{Analyses of Group Regret Bound}
The exchange of information in \texttt{MA-PETS} enhances the sample efficiency, which will be further shown in the derived multi-agent group regret bound. In other words, the difference between single-agent and multi-agent group regret bounds manifests the usefulness of information sharing among agents. 
\subsubsection{Results for Single-Agent Regret Bound}
\label{sec:single-regret-boud}
Before delving into the detailed analyses of multi-agent group regret bound, we introduce the following useful pertinent results derived from single-agent regret bound \cite{UCRL2}.
% 遗憾由在不在置信范围内两部分组成

Consider $M^{(i)}$ as a time-homogeneous and communicating MDP with a policy $\pi^{\prime\ast(i)}$. Meanwhile, by Lemma \ref{lem:optimal_T_bound}, the optimal $T$-step reward $\mathcal{R}^{\ast(i)}_T(M^{(i)})$ can be approximated by the optimistic average reward $\rho^{\ast(i)}$. Therefore, if $\pi^{\prime \ast}(i)$ is implemented in $M^{(i)}$ for $T$ steps, and we define the regret in a single episode $k$ as $\Delta^{(i)}_k:= \sum_{(s, a)}u^{(i)}_k(s, a)\left(T\rho^{\ast(i)}-\bar{r}^{(i)}(s, a)\right)$, which captures the difference between the $T$-step optimal reward $\rho^{\ast(i)}$ and the mean reward $\bar{r}^{(i)}(s, a)$, with $\sum_k u_k^{(i)}(s, a)=N_{k+1}^{(i)}(s, a)$ and $\sum_{(s, a)}N_{k}^{(i)}(s, a)=T$, we have the following lemma.
\begin{lemma}[Eqs. (7)-(8) of \cite{UCRL2}]
\label{lem:ucrl_rb}
Under Lemma \ref{lem:optimal_T_bound}, the differences between the observed rewards $r_t^{(i)}$, acquired when agent $i$ follows the policy $\pi^{\prime\ast(i)}$ to choose action $a_t^{(i)}$ in state $s_t^{(i)}$ at step $t$, and the aforementioned optimistic average reward $\rho^{\ast(i)}$ forms a martingale difference sequence. Besides, based on Hoeffding’s inequality, with a probability at least $1-\frac{\delta}{12(T)^{5/4}}$, the single-agent bound can be bounded as
\begin{align}
\operatorname{Regret}^{(i)}(T)  &= T\mathcal{R}^{\ast(i)}_T(M^{(i)}) - \sum\nolimits_{t=1}^T r_t^{(i)}(s_t,a_t) \nonumber\\
&\leq \sum\nolimits_k \Delta^{(i)}_k +\sqrt{C_1T \log ({8T}/{\delta})},
% &\leq \sum\nolimits_k \left[\Delta_{k,M^{(i)} \in \mathcal{M}^{(i)}_k}^{(i)} +\Delta_{k,M^{(i)} \notin \mathcal{M}^{(i)}_k}^{(i)}\right]\nonumber\\
% &\quad +\sqrt{C_1T \log ({8T}/{\delta})}
\label{eq:single_agent_rb}
\end{align}
where $C_1=\frac{5}{8}$ is a constant.
\end{lemma}

Lemma \ref{lem:ucrl_rb} effectively transforms the cumulative $T$-step regrets into individual episodes of regret. Moreover, the regret bound $\sum_k \Delta^{(i)}_k$ in Lemma \ref{lem:ucrl_rb} can be classified into two categories, depending on whether the true MDP $M^{(i)}$ falls into the scope of plausible MDPs $\mathcal{M}^{(i)}_k$. As for the case where the true MDP $M^{(i)}$ is not included in the set of plausible MDPs $\mathcal{M}^{(i)}_k$, the corresponding regret $\Delta^{(i)}_{k,M^{(i)} \notin \mathcal{M}^{(i)}_k}$ can be bounded by the following lemma. 

\begin{lemma}[Regret with Failing Confidence Intervals, Eqs. (13) of \cite{ortner_online_2012}]
\label{lem:failing_CI_regret}
As detailed in Sec. 5.1 of \cite{ortner_online_2012}, at each step $t$, the probability of the true MDP not being encompassed within the set of plausible MDPs is given by $P\{M^{(i)} \notin \mathcal{M}^{(i)}_k)\}\leq\frac{\delta}{15t^6}$. Correspondingly, the regret attributable to the failure of confidence intervals is
\begin{align}
    \sum\nolimits_k\Delta_{k,M^{(i)} \notin \mathcal{M}^{(i)}_k}^{(i)}\leq\sqrt{T}. 
\end{align}
\end{lemma}

On the other hand, under the assumption that $M^{(i)}$ falls within the set $\mathcal{M}^{(i)}_k$, in light of Eq. (8) from \cite{Ortner2014SelectingNA}, we ascertain that due to the approximation error of the EVI,
\begin{align}
\label{eq:gap_rho}
    \tilde{\rho}^{(i)}_k \geq \rho^{\ast(i)} - \frac{1}{\sqrt{t_k}},
\end{align}
where $\tilde{\rho}^{(i)}_k$ represents the optimistic average reward obtained by the optimistically chosen policy $\tilde{\pi}^{(i)}_k$, while $\rho^{\ast(i)}$ signifies the actual optimal average reward.

Define $\Delta_{k, M^{(i)} \in \mathcal{M}^{(i)}_k}^{r(i)}:=\sum_{s, a} u_{k}^{(i)}(s, a)\bigg(\tilde{r}^{(i)}_k(s, a)-\bar{r}^{(i)}(s, a)\bigg)$ as the accumulated regret in reward estimation over all state-action pairs. $\tilde{r}^{(i)}_k(s, a)$ denotes the maximal possible reward according to a confidence interval similar to \eqref{eq: confidence interval}. Meanwhile, $\Delta_{k, M^{(i)} \in \mathcal{M}^{(i)}_k}^{p(i)}:= \left(\boldsymbol{u}^{(i)}_k\right)^{\top}\bigg(\left(\tilde{\mathcal{P}}^{(i)}_k\left(\cdot \mid s, \tilde{\pi}_k(s)\right)\right)_{s}-\mathcal{I}\bigg) \boldsymbol{w}^{(i)}_k$ quantifies the regret from the estimation of transition probabilities. Additionally, $\boldsymbol{u}^{(i)}_k:=\left( u_k^{(i)}(s,\tilde{\pi}_k(s))\right)_{s}$ denotes the vector representing the final count of state visits by agent $i$ in episode $k$, as determined under the policy $\tilde{\pi}_k(s)$, where the operator $(\cdot)_s$ concatenates the values for all $s\in \mathcal{S}$.  $\left(\tilde{\mathcal{P}}^{(i)}_k\left(\cdot \mid s, \tilde{\pi}_k(s)\right)\right)_{s}$ is the transition matrix of the policy $\tilde{\pi}^{(i)}_k$ in the optimistic MDP $\tilde{M}^{(i)}_k$ computed via EVI and $\mathcal{I}$ typically represents the identity matrix. Furthermore, $\boldsymbol{w}_k^{(i)}$ denotes a modified value vector to indicate the state value range of EVI-iterated episode $k$ specific to agent $i$, where $w_k^{(i)}(s):=\eta_k^{(i)}(s)-\frac{1}{2}\left(\min _{s \in \mathcal{S}^{(i)}} \eta_k^{(i)}(s)+\max _{s \in \mathcal{S}^{(i)}} \eta_k^{(i)}(s)\right)$ and $\eta_{k}^{(i)}(s)$ is the state value given by EVI. We can then decompose the aggregate regret considering $M^{(i)} \in \mathcal{M}^{(i)}_k$, which leads to the result in Lemma \ref{lem:fallin_regret}.

\begin{lemma}[Regret with $M^{(i)} \in \mathcal{M}^{(i)}_k$, Sec. 5.2 of \cite{Ortner2014SelectingNA}]
\label{lem:fallin_regret}
By the assumption $M^{(i)}\in \mathcal{M}^{(i)}_k$ and the gap derived from \eqref{eq:gap_rho}, the regret $\Delta^{(i)}_k$ of agent $i$ accumulated in episode $k$ could be upper bounded by   
\begin{align}
\label{eq:delta_sum}
\Delta_{k,M^{(i)} \in \mathcal{M}^{(i)}_k}^{(i)} % &:= u_{k,M^{(i)} \in \mathcal{M}^{(i)}_k}^{(i)}(s, a)\left(\rho^{\ast, {(i)}}-\bar{r}^{(i)}(s, a)\right)\nonumber\\
& \leq \underbrace{\Delta_{k,M^{(i)} \in \mathcal{M}^{(i)}_k}^{r(i)}}_{(a)} + \underbrace{\Delta_{k,M^{(i)} \in \mathcal{M}^{(i)}_k}^{p(i)}}_{(b)}\nonumber \\
&\quad +2 \sum_{s, a} \frac{u_{k(M{(i)} \in \mathcal{M}^{(i)}_k)}^{(i)}(s, a)}{\sqrt{t_k}},
\end{align}
%w_k^{(i)}:=u_i(s)-\frac{\min _s u_i(s)+\max _s u_i(s)}{2}
\end{lemma}

As previously noted in Section \ref{sec:systemModel}, we posit that the agents' reward functions $\mathcal{R}^i$ are predetermined, often designed based on task objectives. In line with \texttt{MA-PETS}, the regret contributed by the term (a) in \eqref{lem:fallin_regret} can be disregarded. Therefore, we have the following corollary.
\begin{corollary}
\label{cor:fallin_regret}
   By the assumption $M^{(i)}\in \mathcal{M}^{(i)}_k$ and the gap derived from \eqref{eq:gap_rho}, the regret $\Delta^{(i)}_k$ of agent $i$ accumulated in episode $k$ could be upper bounded by   
\begin{align}
\label{eq:delta_sum_no_r}
 \Delta_{k,M^{(i)} \in \mathcal{M}^{(i)}_k}^{(i)}
 \leq &\underbrace{\Delta_{k,M^{(i)} \in \mathcal{M}^{(i)}_k}^{p(i)}}_{(b)} +2 \sum\nolimits_{s, a} \frac{u_{k,M^{(i)} \in \mathcal{M}^{(i)}_k}^{(i)}(s, a)}{\sqrt{t_k}}.
\end{align} 
\end{corollary}

% \begin{align}
% \,   &\Delta_{k,M^{(i)} \in \mathcal{M}^{(i)}_k}^{r(i)}\nonumber\\
% =& \sum_{s, a} v_{k,M^{(i)} \in \mathcal{M}^{(i)}_k}^{(i)}(s, a)\left(\tilde{r}^{(i)}_k(s, a)-\bar{r}^{(i)}(s, a)\right)\nonumber\\
% \leq& \sqrt{C_2 \log \left({2SAT}/{\delta}\right)} 
%    \sum_{(s, a)} \frac{v_{k,M^{(i)} \in \mathcal{M}^{(i)}_k}^{(i)}(s, a)}{\sqrt{N_{t_k}^{(i)+}(s, a)}}
% &\overset{(a)}{\leq} (1+\sqrt{2})\left(\sqrt{14 \log \left(\frac{2SANT}{\delta}\right)}+2\right)  \sum_{i}^N\sum_{(s, a)}\sqrt{N_T^{(i)}(s, a)}
% \label{eq:bound_r}
% \end{align}
% where the constant $C_2=14$ and similarly, $ \Delta_{k,M^{(i)} \in \mathcal{M}^{(i)}_k}^{p(i)}$ can be bounded by
\begin{figure*}[t]
\begin{align}
\sum_k\Delta_{k,M^{(i)} \in \mathcal{M}^{(i)}_k}^{p(i)}
 % =&\sum_k\sum_s u^{(i)}_k(s)\left(\tilde{\mathcal{P}}^{(i)}_k-\mathcal{I}\right) \boldsymbol{w}^{(i)}_k\nonumber\\
 \leq& \underbrace{\sum_k\sum_s\left\| u^{(i)}_k(s,\tilde{\pi}_k(s))\left(\bigg(\tilde{\mathcal{P}}^{(i)}_k\left(\cdot \mid s, \tilde{\pi}_k(s)\right)\right)_{s}-\left(\mathcal{P}^{(i)}_k\left(\cdot \mid s, \pi_k(s)\right)\right)_{s}\bigg)\right\|_1\left\|\boldsymbol{w}_k^{(i)}\right\|_{\infty}}_{\Delta_{k}^{p_1(i)}}\nonumber\\
 +&\underbrace{\sum_k\sum_s \sum_{s^\prime}u^{(i)}_k(s,\tilde{\pi}_k(s))\left(\mathcal{P}^{(i)}_k\left(s^\prime \mid s, \pi_k(s)\right)-\mathcal{I}_{s,s^\prime}\right) w_k^{(i)}(s^\prime)}_{\Delta_{k}^{p_2(i)}}
\label{eq:bound_p}
\end{align}
    \hrulefill
\end{figure*} 
On the other hand, the item (b) (i.e., $\Delta_{k, M^{(i)} \in \mathcal{M}^{(i)}_k}^{(i)}$) in Lemma \ref{lem:fallin_regret} (as well as Corollary \ref{cor:fallin_regret}) can be further decomposed and bounded as \eqref{eq:bound_p}, where, as demonstrated in Eq. (23) of \cite{Ortner2014SelectingNA}, the modified value vector $\boldsymbol{w}_k^{(i)}$ satisfies $\left\|\boldsymbol{w}_k^{(i)}\right\|_{\infty} \leq \frac{D(M^{(i)})}{2}$. In other words, the state value range is bounded by the diameter $D$ of the MDP at any EVI iteration. By slight rearrangement of $\Delta_{k,M^{(i)} \in \mathcal{M}^{(i)}_k}^{(i)}$, we have \eqref{eq:bound_p}, which can be bounded by  the following lemma.

\begin{lemma}[Eqs. (16)-(17) of \cite{UCRL2}]
\label{lem:bound_p}
Contingent on the confidence interval \eqref{eq: confidence interval}, based on our assumption that both $\tilde{M}^{(i)}_k$ and $M^{(i)}$ belong to the set of plausible MDPs $\mathcal{M}^{(i)}_k$, the term $\Delta_{k}^{p_1(i)}$ in \eqref{eq:bound_p} can be bounded as
\begin{align}
    % &\sum_k\sum_s\left\| u^{(i)}_k(s)\left(\tilde{\mathcal{P}}^{(i)}_k-\mathcal{P}^{(i)}_k\right)\right\|_1\left\|\boldsymbol{w}_k^{(i)}\right\|_{\infty}\nonumber\\
    \Delta_{k}^{p_1(i)}
    \leq & D \sqrt{C_2S \log \left({2AT}/{\delta}\right)} \sum_k \sum_{s, a} \frac{u_k^{(i)}(s, a)}{\sqrt{ N_{t_k}^+(s, a)}},
\end{align}
\noindent where $C_1$ denotes a constant. 
Moreover, by applying the Azuma-Hoeffding inequality, the term $\Delta_{k}^{p_2(i)}$ in \eqref{eq:bound_p} can be bounded as
\begin{align}
     % &\sum_k\sum_s u^{(i)}_k(s)\left(\mathcal{P}^{(i)}_k-\mathcal{I}\right) \boldsymbol{w}_k^{(i)}\nonumber\\
     \Delta_{k}^{p_2(i)}
     \leq & D \sqrt{4C_1T \log \left({8T}/{\delta}\right)}+DSA \log _2\left({8T}/{S A}\right),
\end{align}
where $C_2$ denotes a constant. 
\end{lemma}

\subsubsection{Analysis for Group Regret with Limited-Range Communications}
Before analyzing the group regret in the multi-agent system within a constrained communication range, we start with some necessary notations and assume all CAVs constitute a graph. Considering the power graph $\mathcal{G}_{d,k}$, which is derived by selecting nodes from the original graph and connecting all node pairs in the new graph that are at a distance less than or equal to a specified value $d$ in the original graph. The neighborhood graph $\mathcal{G}_{d,k}^{(i)}=\left(\mathcal{F}_{d,k}^{(i)},\mathcal{E}^{(i)}_{d,k}\right)$ represents a sub-graph of $\mathcal{G}_{d,k}$, where $\mathcal{E}^{(i)}_{d,k}$ comprises the set of communication links connecting the agents $i$ and $i^{\prime} \in \mathcal{F}_{d,k}^{(i)}$. We also define the total number of state-action observations for agent $i$, which is calculated as the sum of observations from its neighbors within the communication range $d\in [0, D(\mathcal{G}_k)]$ before episode $k$, denoted as $N_{d,k}^{(i)+}(s, a):=N_{d,k-1}^{(i)+}(s, a) +\sum_{i^{\prime} \in \mathcal{G}_{d,k}^{(i)}} u_k^{(i^{\prime})+}(s, a)$. Given the complication to directly estimate $N_{d,k}^{(i)+}(s, a)$, we introduce the concept of clique cover\cite{CORNEIL1988109} for $\mathcal{F}_{d,k}^{(i)}$, which is a collection of cliques that can cover all vertices of a power graph. Further, let $\mathbf{C}_{d,k}$ denote a clique cover of $\mathcal{G}_{d,k}$ and a clique $\mathcal{C} \in \mathbf{C}_{d,k}$. Meanwhile, $|\mathcal{C}|$ represents the size of the clique $\mathcal{C}$. 
Additionally, the clique covering number $\bar{\chi}\left(\mathcal{G}_{d,k}\right)$ signifies the minimum number of cliques to cover the power graph $\mathcal{G}_{d,k}$  within the communication. Furthermore, characterized by a clique covering number $\bar{\chi}\left(\mathcal{G}_d\right):= \max_{k \in [K]} \{\bar{\chi}\left(\mathcal{G}_{d,k}\right)\}$, $\mathbf{C}_{d}$ constitutes the clique cover for the graph $\mathcal{G}_{d}$.
\begin{lemma}[Eq. (6) of \cite{9867146}]
\label{lemma:clique_number}
Let $\mathbf{C}_{d,k}$ represent the clique covering of the graph $\mathcal{G}_{d,k}$, where the graph $\mathcal{G}_{d,k}$ consists of $I$ nodes, and the cliques within $\mathbf{C}_{d,k}$ have node counts $|\mathcal{C}_{\mathcal{C} \in \mathbf{C}_{d,k}}|$. The minimum clique cover $\bar{\chi}\left(\mathcal{G}_{d,k}\right)$ maintains a consistent relationship  
 $\sum_{\mathcal{C} \in \mathbf{C}_{d,k}}\sqrt{|\mathcal{C}|}=\sqrt{\bar{\chi}\left(\mathcal{G}_d\right)I}$.

\end{lemma}

In this setup, all agents within a clique $\mathcal{C}$ explore in proportion to the clique's size and share the collected samples among themselves. For simplicity, we assume that $\mathcal{G}_{d,k}$ is connected, meaning that all cliques can communicate with one another. Meanwhile, we take $N_{\mathcal{C},k}^+(s, a)$ to be the number of samples exchanged within the clique $\mathcal{C}$ before an episode $k$ that are available for all agents $i \in \mathcal{C}$. In the most challenging scenario, characterized by uniformly random exploration, the algorithm fails to capitalize on any inherent structures or patterns within the environment to potentially refine its learning approach. Consequently, this necessitates a maximum quantity of samples to ascertain an optimal policy, thus giving the worst case. Under these conditions, we present the following lemma.

\begin{lemma}[Theorem 1 of \cite{9867146}]
\label{lemma:clique_sample_number}
In the worst case, where the exploration is uniformly random, the number of samples within each clique can be approximated as $N_{\mathcal{C}, k}(s, a)=|\mathcal{C}|Hk/SA$, where $H$ is the maximum horizon, $K$ is the number of episodes, $S$ and $A$ represents the state and action spaces, respectively.
\end{lemma}
\begin{comment}
\begin{lemma}
Intuitively, each agent in the clique receives no more than $d_{\mathcal{C}}^{\ast}-|\mathcal{C}|$ out-of-clique samples corresponding to step $h$ per episode. 
For simplicity, we take $N_{d,k}^+(s, a)$ to be the smallest number of samples generated by the clique before episode $k$ that are available for all agents $i \in [I]$. $d_{\mathcal{G}_d}^{\text {avg }}:=\left(\sum_{C \in \mathbf{C}_d}\left(d_{\mathcal{C}}^{\ast}-|\mathcal{C}|\right)^{-1}\right)^{-1}$ is a constant that represents the average cross-clique degree structure, and $\bar{\chi}\left(\mathcal{G}_d\right)$ denotes the clique covering number (i.e., minimum number of cliques in a clique covering) of the communication power graph $\mathcal{G}_d$. In the worst case, when exploration is uniformly random, the number of clique samples $N_{\mathcal{C}, k}(s, a)=|\mathcal{C}|^2k/SA$\cite{9867146}. 
\end{lemma}
\end{comment}
% This analysis enables agents to engage in off-policy learning by receiving messages containing the explored experiences of other agents. In this setup, all agents within a clique $\mathcal{C}$ explore in proportion to the clique's size and share the collected samples among themselves. For simplicity, we assume that $\mathcal{G}_d$ is connected, meaning that all cliques communicate with one another, and $d_{\mathcal{C}}^{\ast} > |\mathcal{C}|$.

Based on the results of the single-agent regret bound described in Section \ref{sec:single-regret-boud}, it is ready to analyze the upper bound on the group regret of the multi-agent setting under limited communication range, as shown in Theorem \ref{thm:regret_bound_maucrl}.

\begin{theorem}{(Hoeffding Regret Bound for Parallel MDP with Limited-Range Communications)}
\label{thm:regret_bound_maucrl}
With probability at least $1- \delta$, it holds that for all initial state distributions and after any $T$ steps, the group regret with limited communication range is upper bounded by \eqref{eq:groupregret_maucrl}.
\begin{figure*}[t]
\begin{align}
\operatorname{Regret}_G(T)  &\leq \sqrt{C_1IT \log ({8IT}/{\delta})}+I\sqrt{T}\left[1+(1+\sqrt{2}) \sqrt{SA}\right] +D\sqrt{4C_1 IT \log ({8IT}/{\delta})} \nonumber\\
&\quad +DSAI \log _2\left({8T}/{SA}\right) +(1+\sqrt{2})DS \sqrt{C_2 \bar{\chi}\left(\mathcal{G}_d\right)IAT \log \left({2AT}/{\delta}\right)} 
\label{eq:groupregret_maucrl}
\end{align}
    \hrulefill
\end{figure*} 
\end{theorem}

\begin{proof}
The proof of Theorem \ref{thm:regret_bound_maucrl} stands consistently with that of Lemma \ref{lem:ucrl_rb}. However, it contains significant differences, due to the information exchange-induced distinctive empirical sample size for the transition probabilities. 

In particular, the $T$-step group regret in the MDP settings could be bounded based on \eqref{eq:group_regret} and \eqref{eq:single_agent_rb}, that is, by Lemma \ref{lem:ucrl_rb} with a high probability at least $1-\frac{\delta}{12(IT)^{5/4}}$,
\begin{align}
\operatorname{Regret}_G(T)& = \sum_{i}^I\left[T\rho^{\ast(i)}-\sum_{t=1}^T r_t^{(i)}(s_t,a_t)\right] \nonumber\\
&\leq \sum_{i}^I\sum_k^K \left[\Delta_{k,M^{(i)} \notin \mathcal{M}^{(i)}_k}^{(i)}+\Delta_{k,M^{(i)} \in \mathcal{M}^{(i)}_k}^{(i)}\right]\nonumber\\
&\quad + \sqrt{C_1IT \log ({8IT}/{\delta})}
\label{eq:bound_maucrl}
\end{align}

As elucidated in Lemma \ref{lem:failing_CI_regret}, the portion of group regret $\sum_{i}^I\sum_k \Delta_{k,M^{(i)} \notin \mathcal{M}^{(i)}_k}^{(i)}$ resulting from failed confidence regions can be bounded by $I\sqrt{T}$ with a probability of at least $\frac{\delta}{15t^6}$. 

On the other hand, for each episode $k$, under the premise that the true MDP is encompassed within the set of plausible MDPs ($\mathcal{M}_k^i, \forall k$), the term $\Delta_{k, M^{(i)} \in \mathcal{M}^{(i)}_k}^{(i)}$ is further dissected according to Corollary \ref{cor:fallin_regret}. With Lemma \ref{lem:bound_p}, we have
\begin{align}
\label{eq:ma_delta_p}
&\sum_{i}^I\sum_k^K \Delta_{k,M^{(i)} \in \mathcal{M}^{(i)}_k}^{(i)} \nonumber\\
&\leq \sum_{i}^I \sum_{k=1}^K\Delta_{k,M^{(i)} \in \mathcal{M}^{(i)}_k}^{p(i)}  +2 \sum_{i}^I\sum_k^K\sum_{s, a} \frac{u_k^{(i)}(s, a)}{\sqrt{N_{k}^{(i)+}(s, a)}}\nonumber\\
&\leq D \sqrt{C_2S \log \left({2AT}/{\delta}\right)} \sum_{k=1}^K \sum_{\mathcal{C} \in \mathbf{C}_{d,k}} \sum_{i \in \mathcal{C}}  \sum_{s, a} \frac{u_k^{(i)}(s, a)}{\sqrt{ N_{\mathcal{C},k}^+(s, a)}} \nonumber\\
&+DSAI \log _2\left({8T}/{S A}\right) 
+D \sqrt{4C_1TI \log \left({8TI}/{\delta}\right)} \nonumber\\
& +2 \sum_{i}^I\sum_{k=1}^K\sum_{s, a} \frac{u_{k,M^{(i)} \in \mathcal{M}^{(i)}_k}^{(i)}(s, a)}{\sqrt{N_{k}^{(i)+}(s, a)}}.
\end{align}

Leveraging the definition that $N_{k+1}^{(i)+}(s, a) = \max \left\{1, \sum_{i=1}^k u_{k}^{(i)}(s, a)\right\}$, % where $0 \leq u_{k}^{(i)}(s, a) \leq N_{k}^{(i)+}(s, a)$  and $\sum_{(s,a)}N_{k}^{(i)+}(s,a)=kH$, 
and applying Jensen’s inequality and the inequality in Lemma \ref{lem: inequality} in Appendix, we obtain the following results
\begin{align}
    \sum_{i}^I\sum_{k=1}^K\sum_{s, a} \frac{u_{k,M^{(i)} \in \mathcal{M}^{(i)}_k}^{(i)}(s, a)}{\sqrt{N_{k}^{(i)+}(s, a)}} \leq(\sqrt{2}+1)I\sqrt{SAT}.
\end{align}
Furthermore, by Lemma \ref{lemma:clique_sample_number}, $\sum_{(s,a)}N_{\mathcal{C}, k}^+(s, a)=|\mathcal{C}|Hk$. Thus, we have
\begin{align}
\label{eq:modified_visit}
   &\sum_{k=1}^K \sum_{\mathcal{C} \in \mathbf{C}_{d,k}}    \sum_{s, a} \frac{\sum_{i \in \mathcal{C}}u_k^{(i)}(s, a)}{\sqrt{ N_{\mathcal{C},k}^+(s, a)}}\nonumber\\
    &\leq \sum_{k=1}^K \sum_{\mathcal{C} \in \mathbf{C}_{d,k}} \sum_{s, a}  \frac{u_{\mathcal{C},k}(s, a)}{\sqrt{N_{\mathcal{C},k}^+(s, a)}}\nonumber\\
& \leq \sum_{\mathcal{C} \in \mathbf{C}_d} (\sqrt{2}+1)\sqrt{|\mathcal{C}|SAKH}\nonumber\\
&= (\sqrt{2}+1)\sqrt{\bar{\chi}\left(\mathcal{G}_d\right)ISAT},
\end{align}
where $u_{\mathcal{C},k}(s, a):= \sum_{i \in \mathcal{C}} u_k^{(i)}(s, a)$  represents the frequency of visits within the clique $\mathcal{C}$  to a state-action pair after the communications at the end of episode $k$.

In summary, we conclude the proof of Theorem \ref{thm:regret_bound_maucrl} and establish that the total regret is bounded by \eqref{eq:groupregret_maucrl} with a probability of $1-\frac{\delta}{4 T^{5/4}}$. When considering all values of $T=2, \ldots$, it is evident that this bound holds simultaneously for all $T \geq 2$ with a probability of at least $1-\delta$.
\end{proof}
Theorem \ref{thm:regret_bound_maucrl} unveils the group regret bound and demonstrates that concerning $T$, a sub-linear increase in the group regret bound is attained in \eqref{eq:groupregret_maucrl}. On the contrary, in the standard non-communicative reinforcement learning setting for the same type of algorithm, the group regret is essentially equal to the sum of the single-agent regrets, as implied in Lemma \ref{lem:ucrl_rb}.
 
\section{Experimental Settings and Numerical Results}
\label{sec: numerical experiments}
\subsection{Experimental Settings}
\begin{figure*}[!t]
	\centering
        \captionsetup{font=small}
	\subfigure[``Figure Eight" loop]{
		\begin{minipage}[b]{0.21\linewidth}
			\centering
			\includegraphics[height=1.3in]{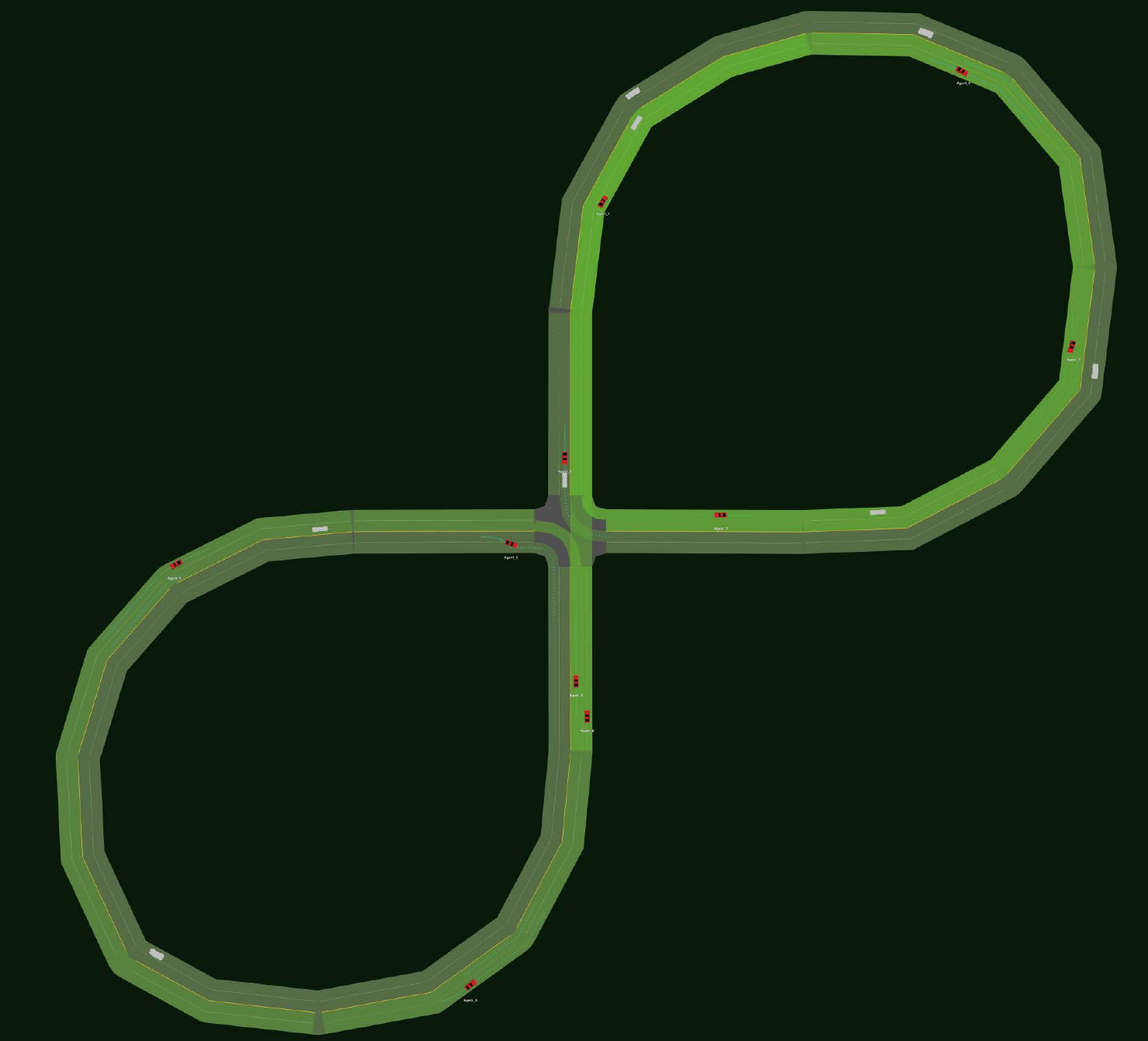}
            \label{subfig:aerial_view}
		\end{minipage}
	}%
	\subfigure[Single-lane ``Unprotected Intersection'']{
		\begin{minipage}[b]{0.34\linewidth}
			\centering
			\includegraphics[height=1.3in]{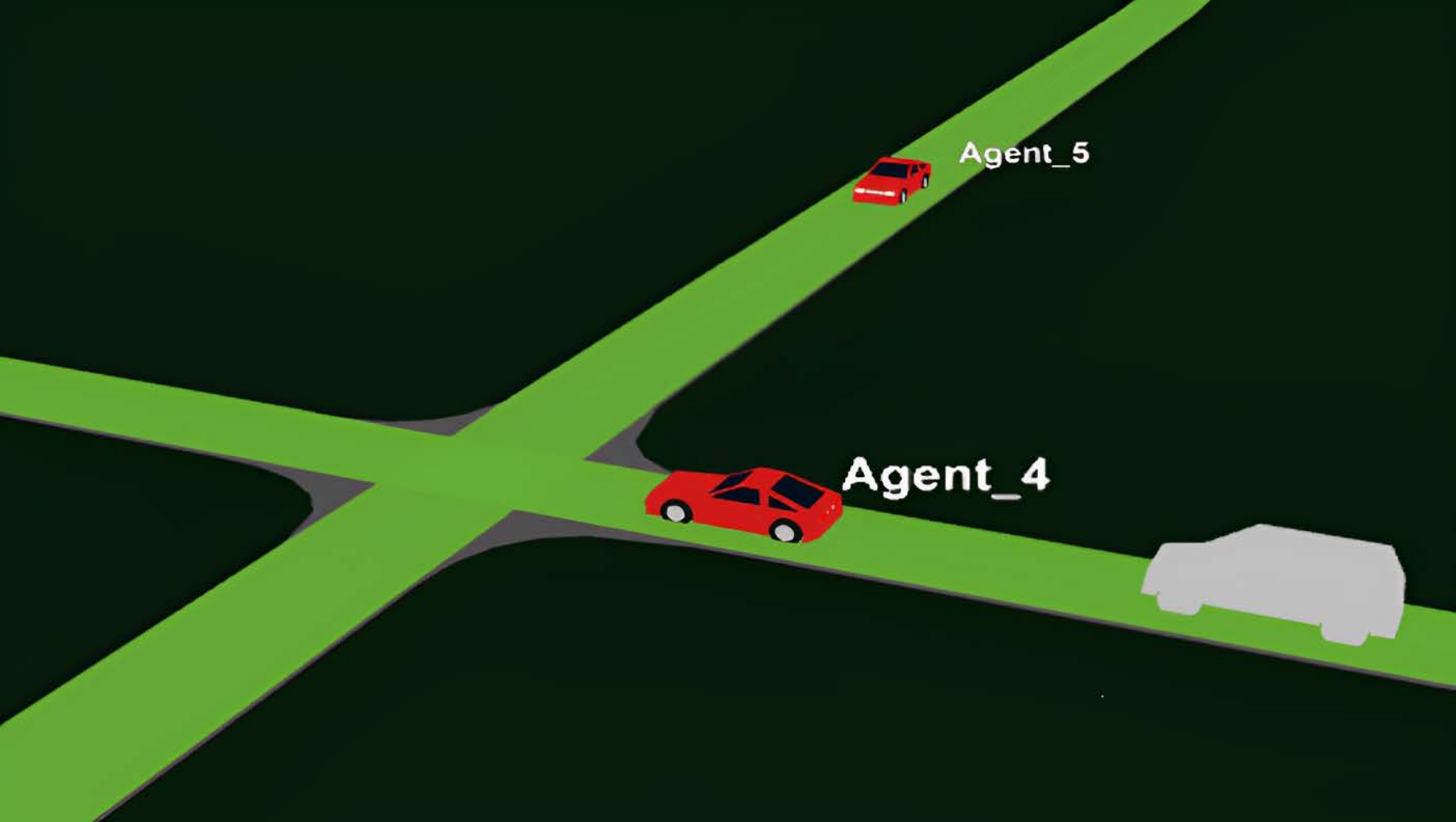}
            \label{subfig:regional_view_single_lane}
		\end{minipage}
	}%
    \subfigure[Multi-lane ``Unprotected Intersection'']{
		\begin{minipage}[b]{0.37\linewidth}
			\centering
			\includegraphics[height=1.3in]{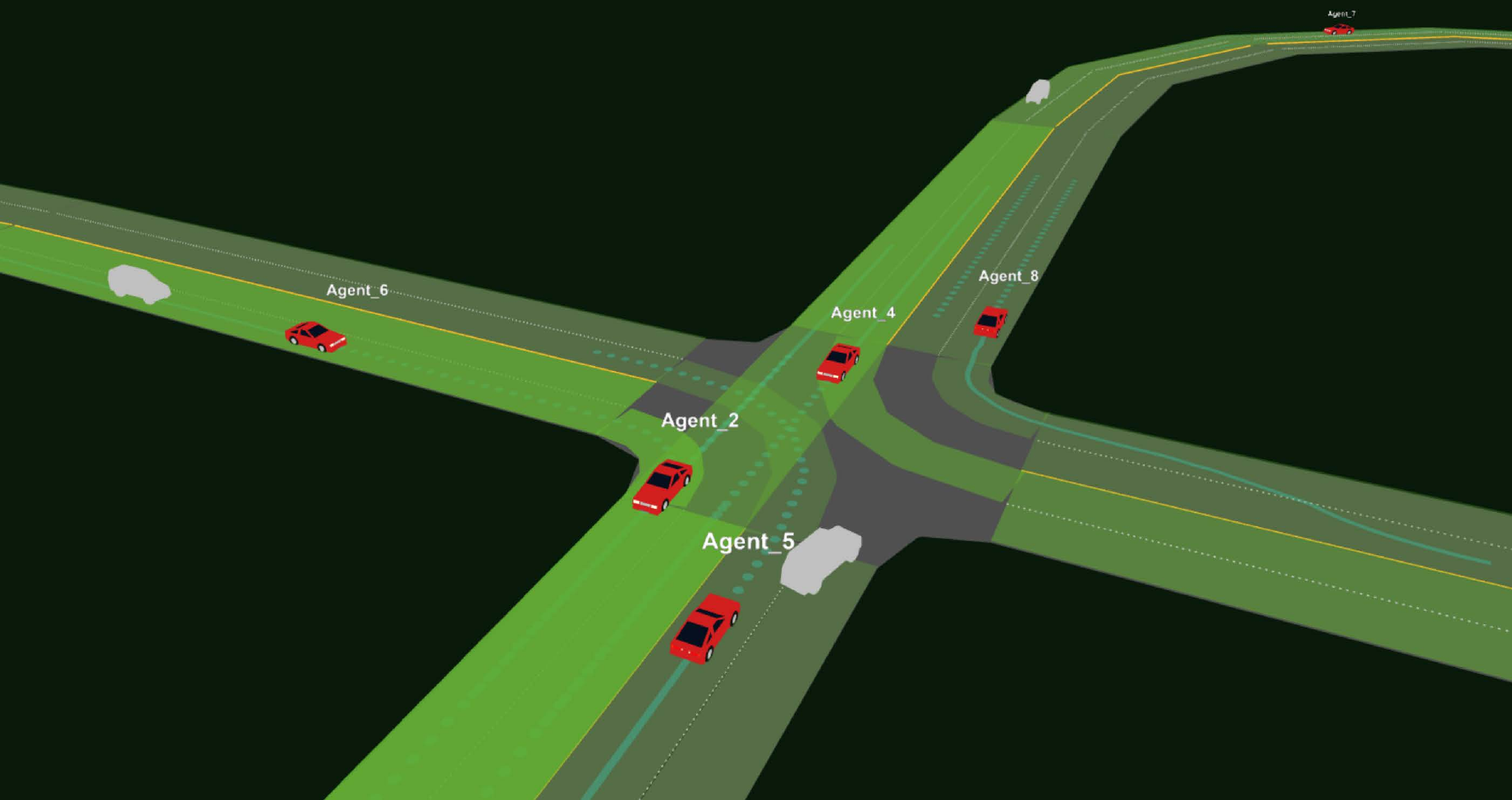}
            \label{subfig:regional_view_multi_lane}
		\end{minipage}
	}%
	\centering
	\caption{The ``Unprotected Intersection'' scenario in the closed single-lane and multi-lane ``Figure Eight" loop for Simulations. (a) presents an aerial view of the ``Figure Eight" loop, while (b) and (c) provide the regional enlarged view of the single-lane and multi-lane ``Unprotected Intersection'', respectively.}
 \label{fig3}
\end{figure*}

In this section, we evaluate the performance of \texttt{MA-PETS} within the domain of autonomous vehicle control and demonstrate the superiority of our proposed algorithm over several other state-of-the-art RL methods, including \texttt{FIRL}~\cite{FIRL},  \texttt{SVMIX}~\cite{SVMIX}, \texttt{MAPPO}~\cite{MAPPO}, \texttt{MADDPG}~\cite{MADDPG}, \texttt{DQN}~\cite{mnih2015human}, \texttt{SAC}~\cite{haarnoja2018soft}. Specifically, we run our experiments using the CAV simulation platform SMARTS~\cite{zhou2020smarts} and select the ``Unprotected signal-free Intersection'' scenario in the closed single-lane and multi-lane ``Figure Eight" loop, which is a typical mixed autonomy traffic scenario as illustrated in Fig. \ref{fig3}, for the evaluation. In these experiments, we deploy a fixed number (i.e., $I$) CAVs respectively, managed by our \texttt{MA-PETS} algorithm, and a random number $I_\text{hv}$ of human-driven vehicles (HVs) controlled by the environment within the SMARTS framework. All CAVs will start from a one-way lane with an intersection, and drive circularly by passing through the intersection while avoiding collisions and congestion. Following the MDP defined in Section \ref{sec: preliminaries}, the corresponding MDP for CAVs in “Unprotected Intersection” is defined as below.
\begin{itemize}
    \item \emph{State and Observation}: Except for information about the vehicle itself, each CAV can only observe the information of the vehicle ahead and behind. Hence, it has the information about current state containing the velocity $v_t \in \mathbb{R}$, position $z_t=\left(x_t, y_t\right) \in \mathbb{R}^2$ of itself, the speed and distance of the vehicles ahead and behind $v_{t, a},v_{t,e},l_{t, a},l_{t,e} \in \mathbb{R}$. Hence, as mentioned earlier, the state of vehicle $i$ can be represented as $s_{t}^{(i)}=\left(v_t^{(i)},x_t^{(i)}, y_t^{(i)},v_{t,a}^{(i)},v_{t,e}^{(i)},l_{t,a}^{(i)},l_{t,e}^{(i)} \right)$.
     \item \emph{Action}: As each CAV is controlled via the target velocity $v \in \mathbb{R}$ and the special actions indicating whether to make a lane change $c \in \{0,1\}$, decided by itself, the action of vehicle $i$ can be represented as $a_t^{(i)}=\{v_t^{(i)},c_t^{(i)}\}$. Besides, for the single-lane scenario, $c$ nulls.
    \item \emph{Reward}: CAVs aim to maintain a maximum velocity based on no collision. Accordingly, the reward function $\mathcal{R}^{(i)}$ can be defined as 
    \begin{align}
        \label{eq:rewardfunction}
        r_t^{(i)} =\mathcal{R}^{(i)}\left(s_t^{(i)},a_t^{(i)}\right) =v_t^{(i)}+v_{t,a}^{(i)}+v_{t,e}^{(i)}+\beta,
    \end{align}
    where the extra term $\beta$ imposes a penalty term on collision, that is, $\beta=-10$ if a collision occurs; and it nulls otherwise. 
\end{itemize}

Notably, given the definition of the reward function in \eqref{eq:rewardfunction}, a direct summation of the rewards corresponding to all agents could lead to the computation duplication of some vehicles, thus misleading the evaluation. Therefore, we develop two evaluation metrics from the perspective of system agility and safety. In particular, taking account of the cumulative travel distance ${\Lambda}x_t^{(i)}$ of vehicle $i$ at time-step $t$ and the lasting-time $t_c$ before the collision, we define 
\begin{align}
\label{eq:utilityMetric}
    {\rm Agility}\left(K\right)=\mathbb{E}_{\boldsymbol{\pi} \sim \tilde{\mathcal{P}}_k}\frac{1}{I} \sum\nolimits_{i=1}^I \frac{1}{t_c^{(i)}}\sum\nolimits_{t=0}^{t_c^{(i)}} {\Lambda}x_t^{(i)}
\end{align} 
and 
\begin{align}
\label{eq:safetyMetric}
    {\rm Safety}\left(K\right) =\mathbb{E}_{\boldsymbol{\pi} \sim \tilde{\mathcal{P}}_k}\frac{1}{I} \sum\nolimits_{i=1}^I \frac{t_c^{(i)}}{H},
\end{align} 
where $\tilde{\mathcal{P}}_k$ ($\forall k\in[K]$) refers to a learned dynamics model until the episode $k$.

In our configuration, the number of episodes is $K=15$, and in the event of no collisions, the length of each episode can reach up to $H=200$ for single-lane scenarios and $H=400$ for multi-lane scenarios, respectively. All of our experiments are conducted on the NVIDIA GeForce RTX 4090 with $5$ independent simulations. Moreover, typical parameters, including the simulation environment setup, deep neural network, and \texttt{MA-PETS}, are summarized in Table \ref{tb: setting}.

\begin{table}[h!]
    \caption{\centering{List of Key Parameter Settings for the Simulation.}}%居中标题
    \begin{center}
    \label{tb: setting}
    \begin{threeparttable} 
            \begin{tabular}{m{4.8cm} |m{3.2cm} } 
            \hline \textbf{Parameters Description} & \textbf{Value} \tabularnewline
            \hline \textbf{Simulator} & \tabularnewline
            Discrete-time step  &  $\tau =0.1 \mathrm{~s}$ \tabularnewline
            Vehicle numbers  & $I=8$ \text{or} $22$ \tabularnewline
            Human-driving Vehicles & $I_\text{hv} \in [8,20]$ \tabularnewline
            Distance of ``Figure Eight'' scenario  & $480\mathrm{~m}$ \tabularnewline
            Maximum velocity of ``Figure Eight'' scenario & $v_{\max } = 13.89 \mathrm{~m} / \mathrm{s}$ \tabularnewline
            Minimum velocity of ``Figure Eight'' scenario & $v_{\min }= 0 \mathrm{~m} / \mathrm{s}$ \tabularnewline
            \hline \textbf{MA-PETS} & \tabularnewline
             Speed per vehicle at time-step $t$ & $v_t \in [0,13.89]\mathrm{~m} / \mathrm{s}$ \tabularnewline
             Position per vehicle at $t$ & $z_t \in \mathbb{R}^2$ \tabularnewline
             Target velocity per vehicle at $t$ & $\bar{v}_{t} \in [0,13.89]\mathrm{~m}$ \tabularnewline
             Speed of vehicles \underline{a}head and b\underline{e}hind & $v_{t,a},v_{t,e} \in [0,13.89]\mathrm{~m} / \mathrm{s}$ \tabularnewline
             Distance of vehicles \underline{a}head and b\underline{e}hind & $l_{t,a},l_{t,e}\in [0,75]\mathrm{~m}$ \tabularnewline
             Travel distance per vehicle at $t$ & ${\Lambda}x_t^{(i)} \in \mathbb{R}$ \tabularnewline
             No. of episodes & $K=15$ \tabularnewline
             Maximal length per episode & $H=200$ \text{or} $400$ \tabularnewline
             Communication range  & $d \in [0,200]\mathrm{~m}$ \tabularnewline
             Horizon of MPC  & $w=25$ \tabularnewline
             Candidate action sequences in CEM & $Q=400$ \tabularnewline
             Elite candidate action sequences in CEM & $X=40$ \tabularnewline
             No. of particles & $P=20$ \tabularnewline
             Max iteration of CEM & $Y=5$ \tabularnewline
             Proportion of elite candidate action sequences & $\alpha=10$\tabularnewline
             Accuracy of CEM Optimizer & $\epsilon=0.001$ \tabularnewline
             Seed number & $5$ \tabularnewline
             \hline \textbf{Probabilistic Ensemble Neural Network} & \tabularnewline
             No. of Ensembles for dynamics Model & $B=5$ \tabularnewline
             Hidden layer number & $3$ \tabularnewline
             Hidden units number & $300$ or $400$ \tabularnewline
             Buffer length & $2, 048$ \tabularnewline
             Batch size  & $128$ \tabularnewline
             Epochs  & $5$ \tabularnewline
             Learning rate  & $lr=0.001$ \tabularnewline
             Optimizer & Adam \tabularnewline
            \hline
            \end{tabular}
    \begin{tablenotes}
        \footnotesize
        \item $\ast$ The parameters setting used only by \texttt{MA-PETS}.
      \end{tablenotes}
  \end{threeparttable}
  \end{center}
\end{table}
\subsection{Numerical Results}
Firstly, we evaluate the performance of \texttt{MA-PETS} for CAVs with $I =8$ and communication range $d=100$ in the single-lane scenario, and present the performance comparison with MFRL algorithms in Fig. \ref{utility_single_lane}. Besides, we testify the real-time performance for each step of the $10$-th episode in the simulation process in Fig. \ref{real-time}. It can be observed from Fig. \ref{utility_single_lane} and Fig. \ref{real-time} that our algorithm \texttt{MA-PETS} significantly yields superior outcomes than the others in terms of both agility and safety. More importantly, \texttt{MA-PETS} converges at a faster pace.

\begin{figure}[!h]
	\centering  %图片全局居中
         \subfigcapskip=-5pt %设置子图与子标题之间的距离
        % \subfigbottomskip=-5pt %设置子图之间的距离
	\subfigure[Agility]{
		\includegraphics[width=0.48\linewidth]{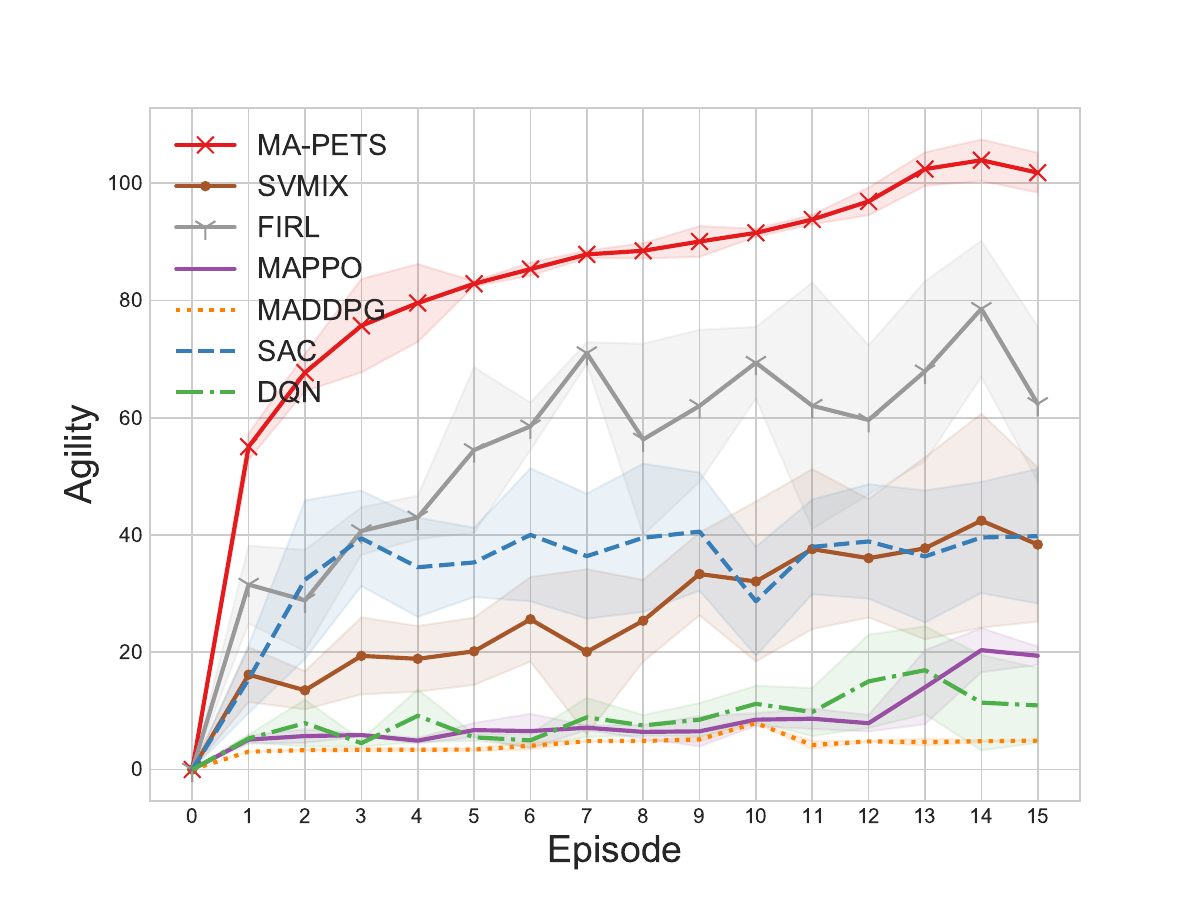}}
	\subfigure[Safety]{
		\includegraphics[width=0.48\linewidth]{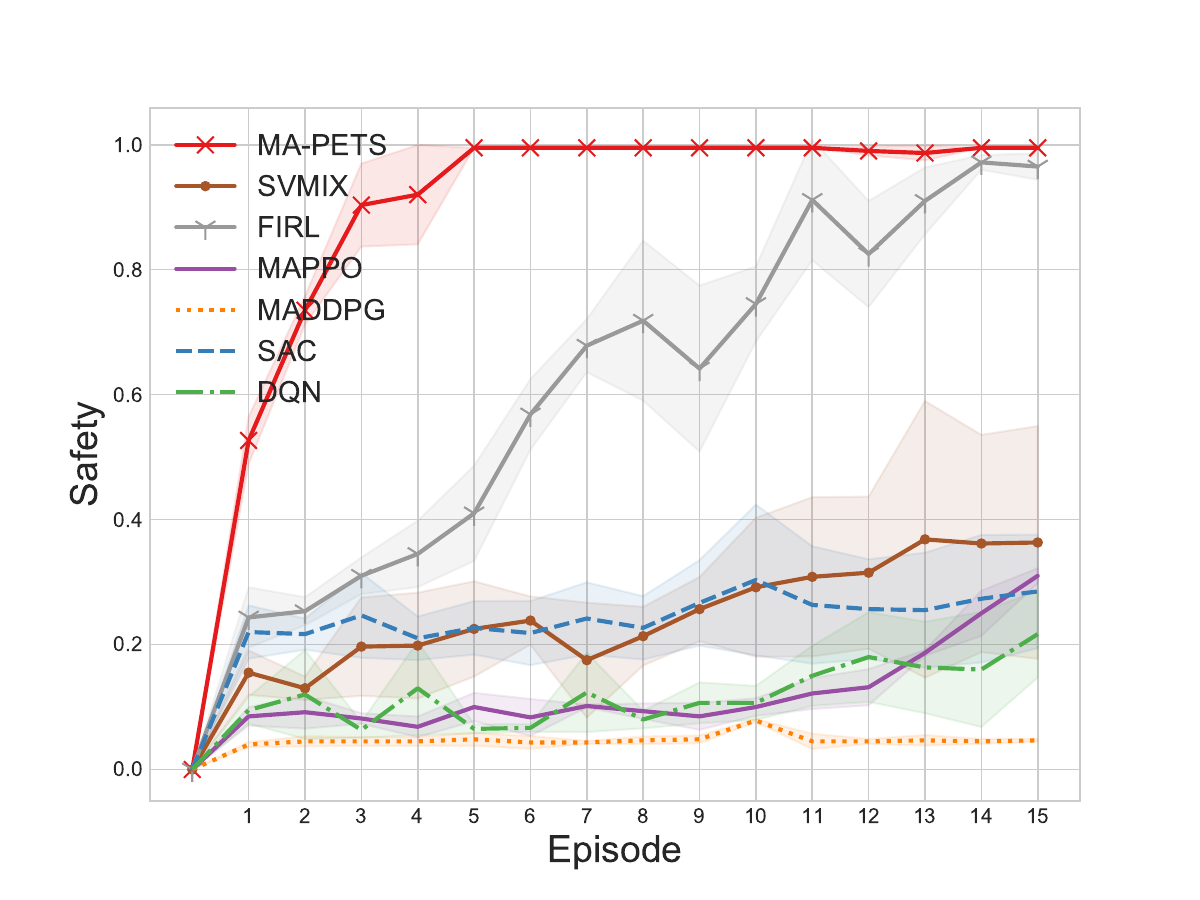}}
	\caption{Comparison of utility in the single-lane ``Unprotected Intersection'' scenario.}
     \vspace{-15pt}
\label{utility_single_lane}
\end{figure}

\begin{figure}[!h]
	\centering  %图片全局居中
         \subfigcapskip=-5pt %设置子图与子标题之间的距离
         % \subfigbottomskip=-5pt %设置子图之间的距离
	\subfigure[Agility]{
		\includegraphics[width=0.48\linewidth]{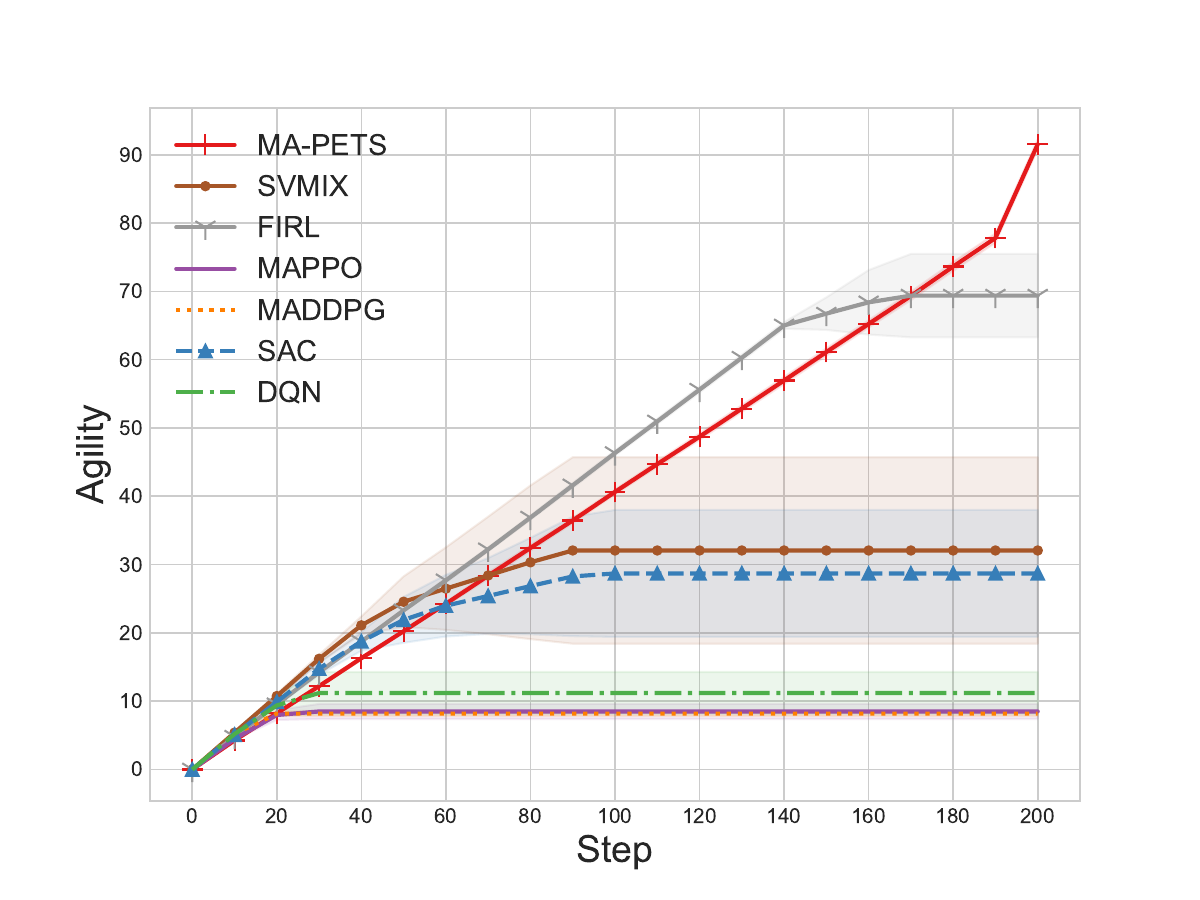}}
	\subfigure[Safety]{
		\includegraphics[width=0.48\linewidth]{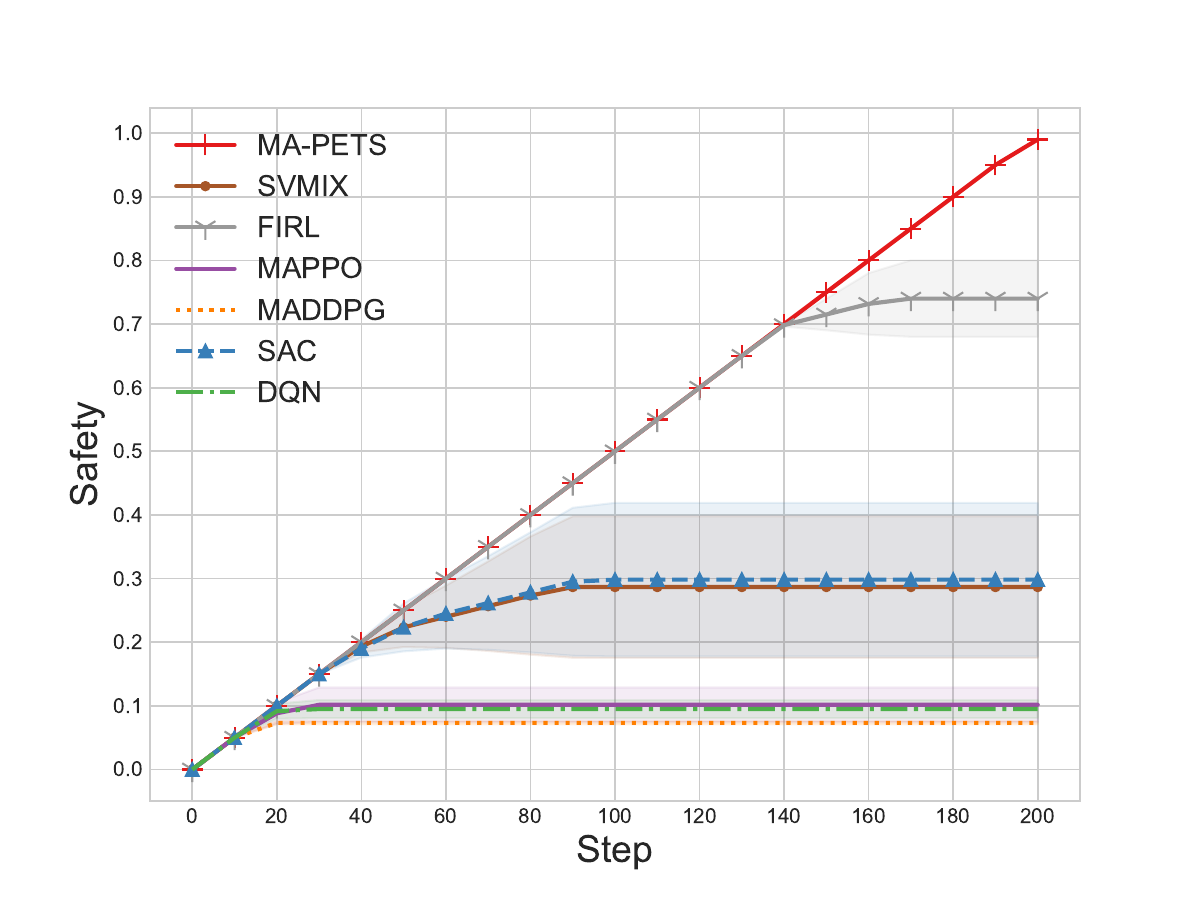}}
	\caption{Real-time utility comparison for each step of the $10$-th episode under the single-lane ``Unprotected Intersection'' scenario.}
     \vspace{-15pt}
\label{real-time}
\end{figure}

Building upon these single-lane results, we further increase the complexity of the simulation scenarios by establishing a multi-lane intersection without traffic lights and utilizing $I=22$ CAVs controlled by our algorithm \texttt{MA-PETS}, as depicted in Fig. \ref{utility_multi_lane}. In this setting, we perform a comparative evaluation of our algorithm against a suite of state-of-the-art RL methods with a communication range $d=100$. As can be observed from Fig. \ref{utility_multi_lane}, while our algorithm may initially exhibit a marginally reduced efficacy in terms of collision avoidance relative to \texttt{MAPPO}, \texttt{SVMIX}, and \texttt{SAC}, it nearly achieves a collision-free state by the $13$-th episode. Simultaneously, our algorithm \texttt{MA-PETS} notably excels in agility and achieves convergence at a significantly accelerated pace compared to other algorithms.

\begin{figure}[!t]
	\centering  %图片全局居中
         \subfigcapskip=-5pt %设置子图与子标题之间的距离
        % \subfigbottomskip=-5pt %设置子图之间的距离
	\subfigure[Agility]{
		\includegraphics[width=0.48\linewidth]{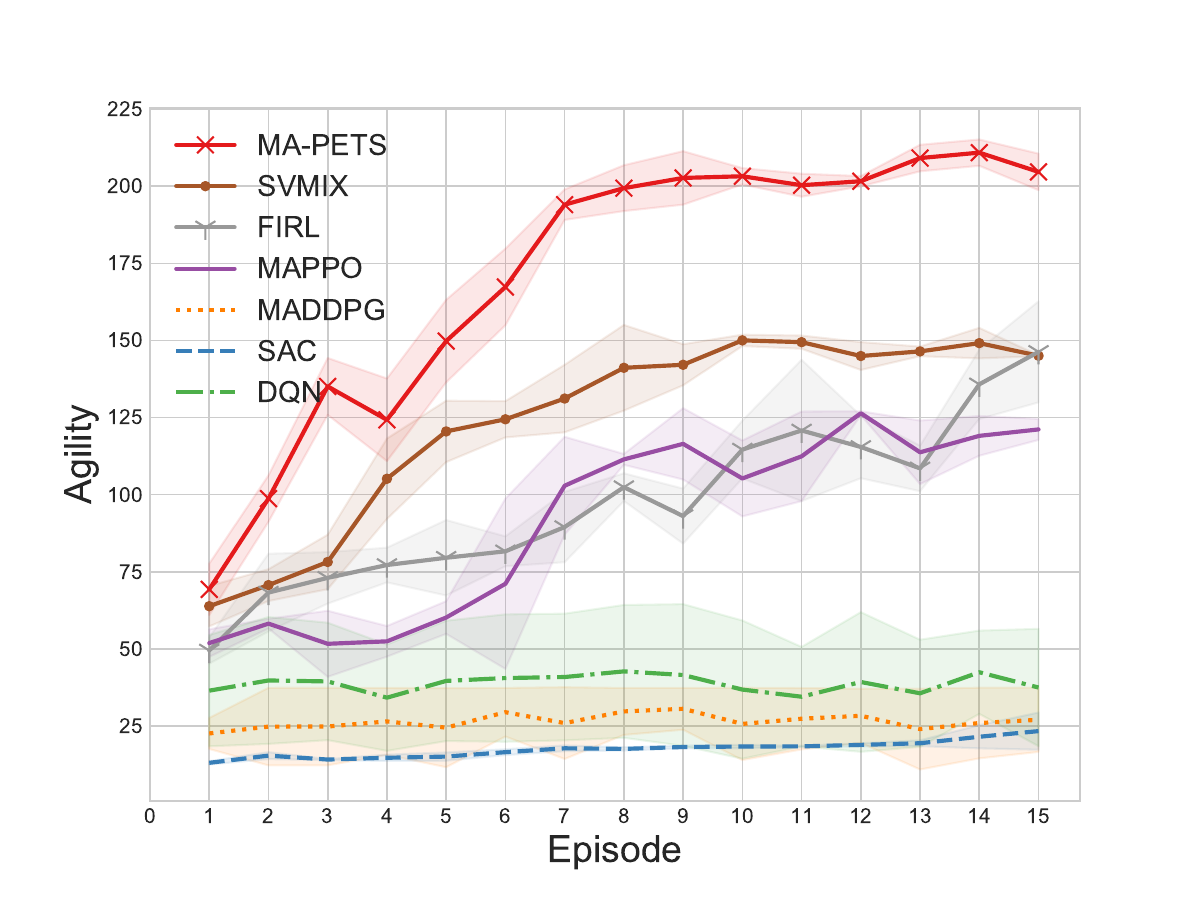}}
	\subfigure[Safety]{
		\includegraphics[width=0.48\linewidth]{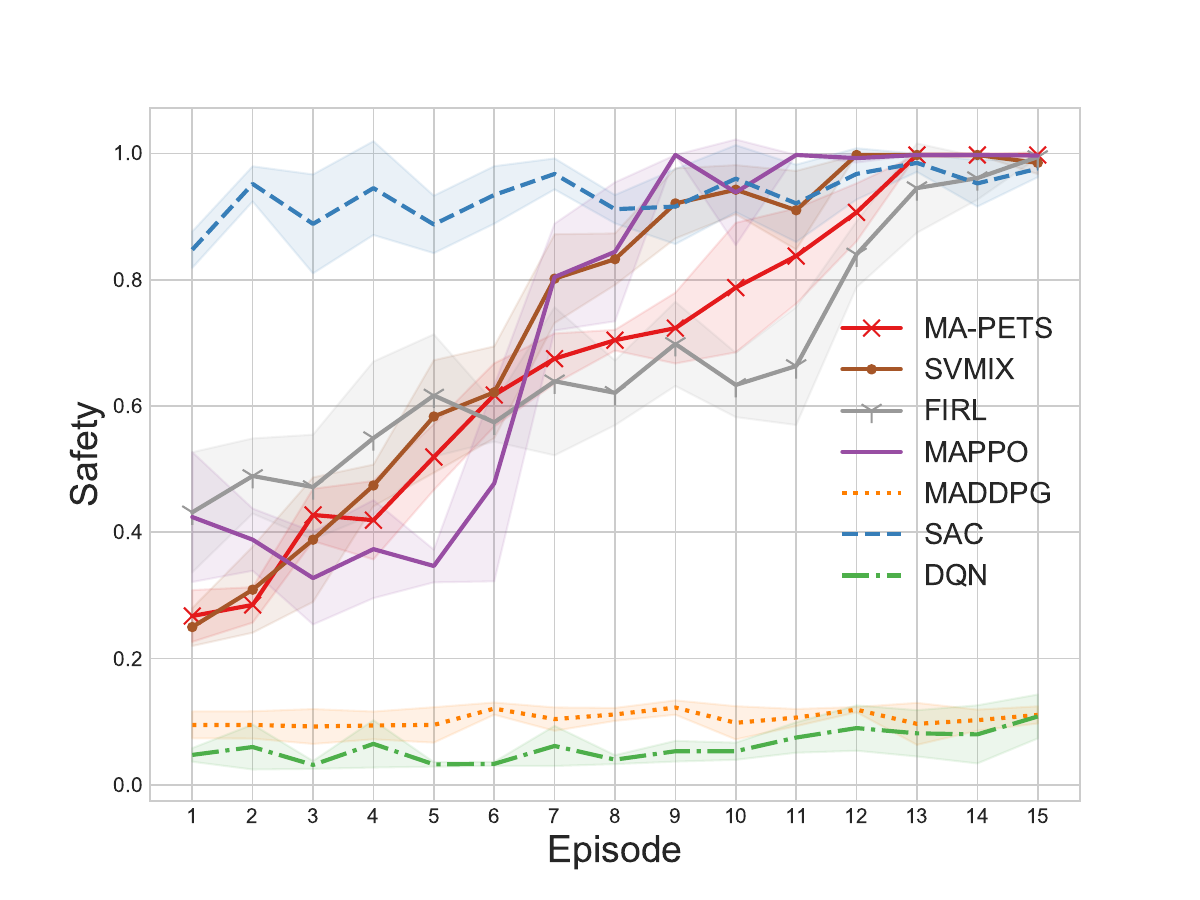}}
	\caption{Comparison of utility in the multi-lane ``Unprotected Intersection'' scenario.}
     \vspace{-5pt}
\label{utility_multi_lane}
\end{figure}

To validate the rationality of Assumption \ref{assump: quantization} and the feasibility of \texttt{MA-PETS} after quantization, we conduct supplementary simulation experiments. We compare the performance of \texttt{MA-PETS} using the tile coding method \cite{sutton2018reinforcement} to discretize our six-dimensional continuous state space and two-dimensional continuous action space with \texttt{MA-PETS} operating in continuous state and action spaces within a single-lane ``Figure Eight" loop scenario, depicted in Fig.\ref{simulation_discretization}. It can be observed that although the discretized \texttt{MA-PETS} algorithm after quantization exhibits larger fluctuations during the learning process, with the increase in the number of training episodes (e.g., post-$15$ episodes), the \texttt{MA-PETS} with discretization can also achieve performance similar to the continuous \texttt{MA-PETS} in terms of agility and safety. This demonstrates that \texttt{MA-PETS} remains effective even in a discretized setting after quantization.
\begin{figure}[!t]
    \centering
    \subfigcapskip=-5pt %设置子图与子标题之间的距离
        % \subfigbottomskip=-5pt %设置子图之间的距离
    \subfigure[Agility]{
		\includegraphics[width=0.48\linewidth]{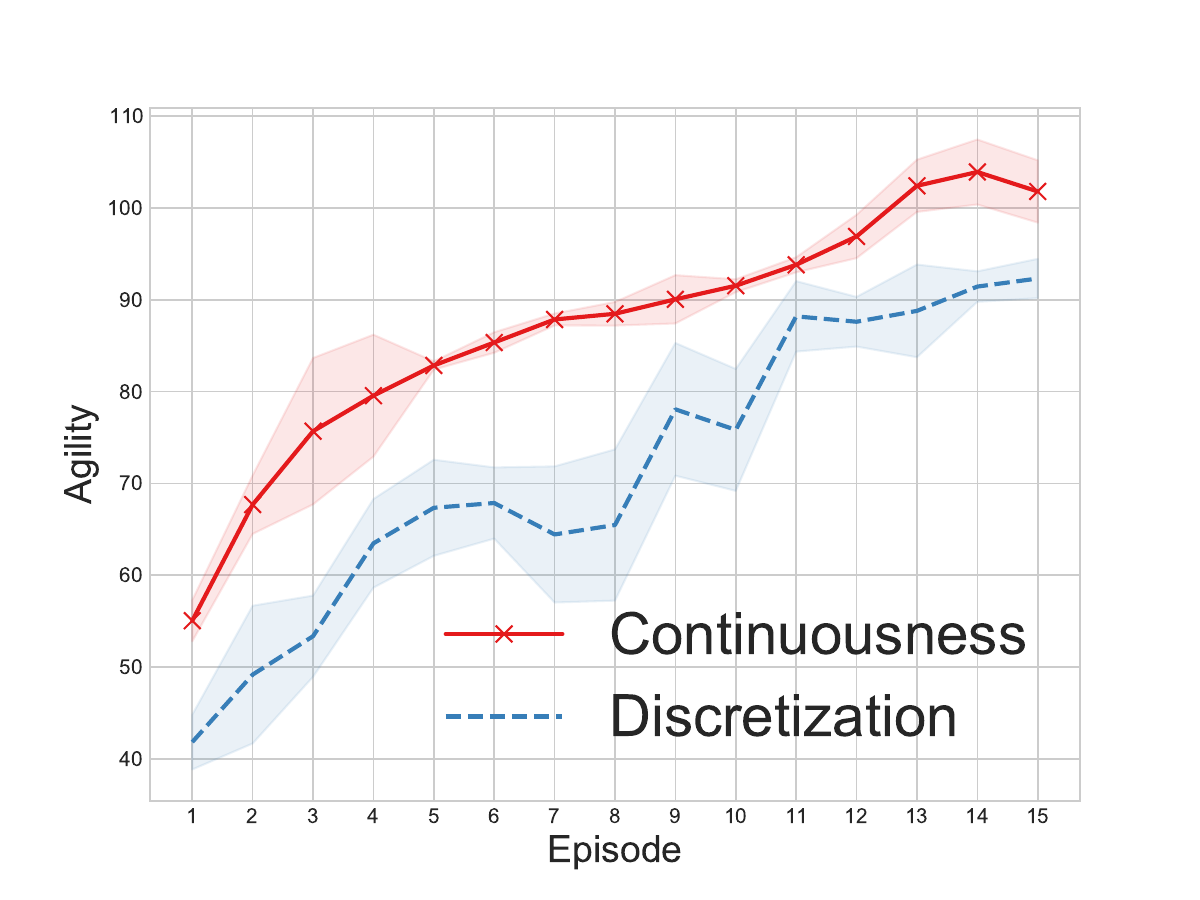}}
  \subfigure[Safety]{
		\includegraphics[width=0.48\linewidth]{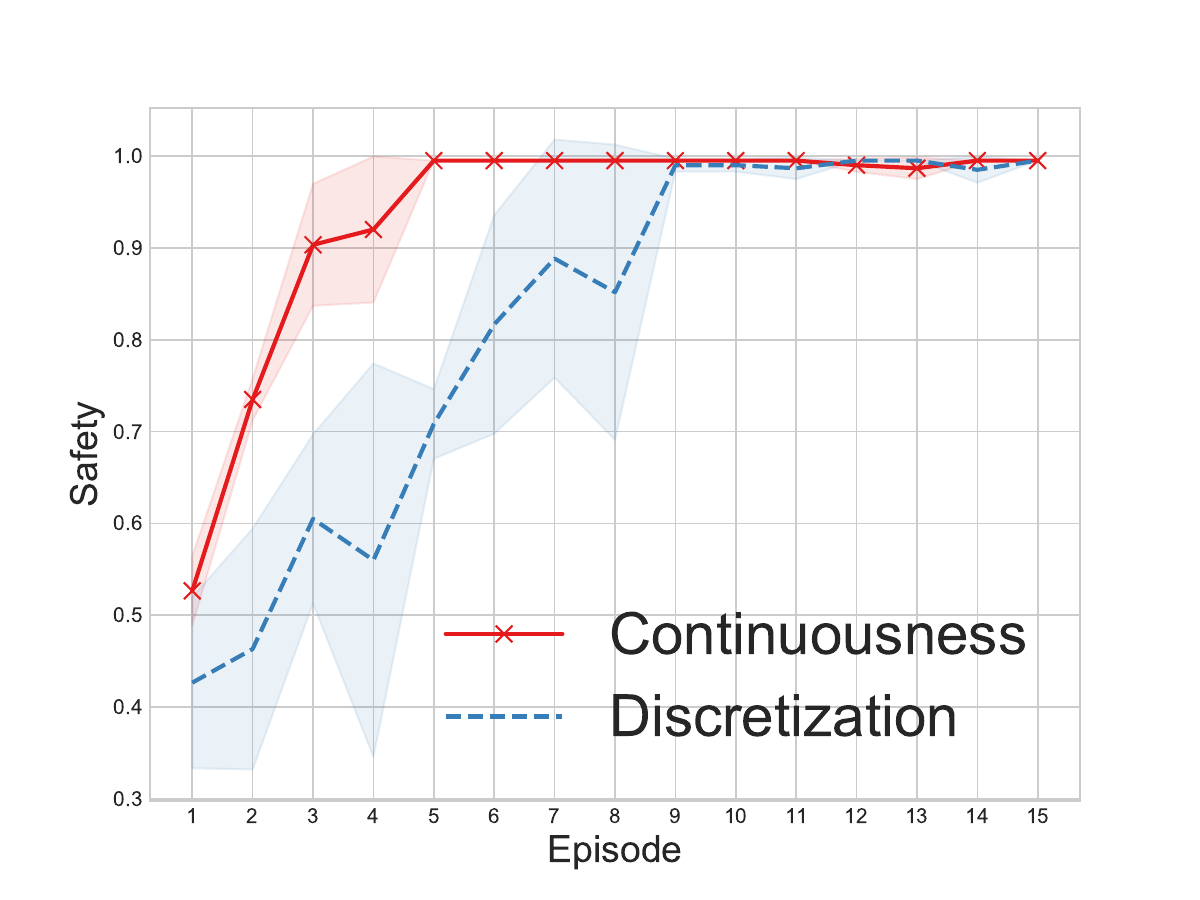}}
    \caption{ Performance comparison of the \texttt{MA-PETS} with continuousness and discretization at the single-lane “Unprotected Intersection” scenario.}
    \vspace{-15pt}
    \label{simulation_discretization}
\end{figure}

Furthermore, we show the performance of \texttt{MA-PETS} concerning different values of communication range $d$ by varying from $0$ to $200$ in Fig. \ref{fig_d} and Table \ref{tb: communications overhead}. It can be observed from Table \ref{tb: communications overhead} that consistent with our intuition, the increase in communication range leads to a significant boost of the communications overheads. Meanwhile, as depicted in Fig. \ref{fig_d}, it also benefits the learning efficiency of the CAVs, thus greatly upgrading agility and safety. On the other hand, Fig. \ref{fig_d} also unveils that when the communication range increases to a certain extent, a further increase of $d$ contributes trivially to the agility and safety of learned policies. In contrast, Table \ref{tb: communications overhead} indicates an exponential increase in the average communications overhead. In other words, it implies a certain trade-off between the learning performance and communication overheads. Consistent with the group regret bound detailed in Section \ref{sec: convergence proof}, we further investigate the numerical interplay between the minimum number of clique covers $\bar{\chi}\left(\mathcal{G}_{d,k}\right)$ and communication distance $d$ in $15$ independent trials. It can be observed from Fig. \ref{clique_cover} that while the minimum number of clique covers remains constant beyond a specific communication range threshold, a generally inverse correlation exists between the minimum clique cover number and the communication radius, which aligns with both our theoretical proofs and the experimental findings presented in Fig. \ref{fig_d}.
\begin{figure}
	\centering  %图片全局居中
         \subfigcapskip=-5pt %设置子图与子标题之间的距离
         % \subfigbottomskip=-5pt %设置子图之间的距离
	\subfigure[Agility]{
		\includegraphics[width=0.48\linewidth]{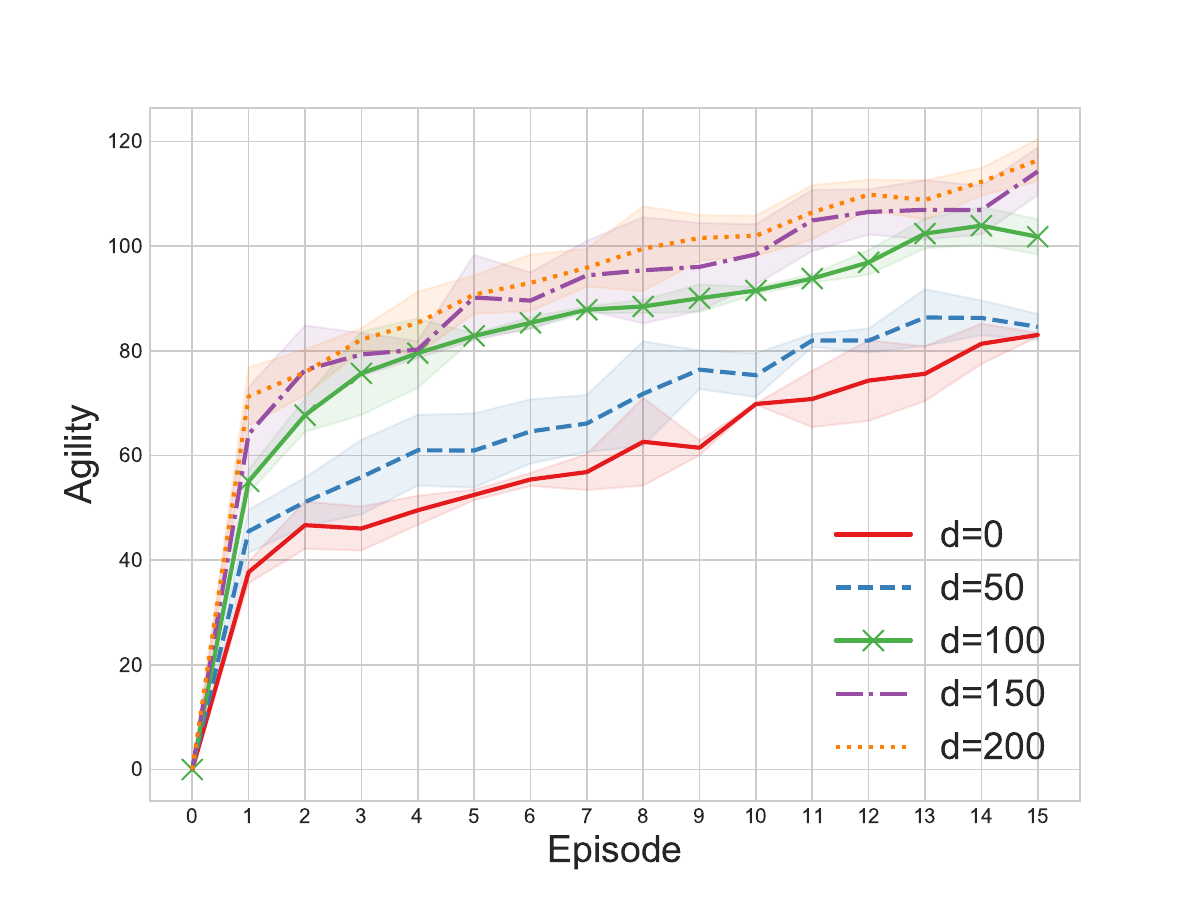}}
	\subfigure[Safety]{
		\includegraphics[width=0.48\linewidth]{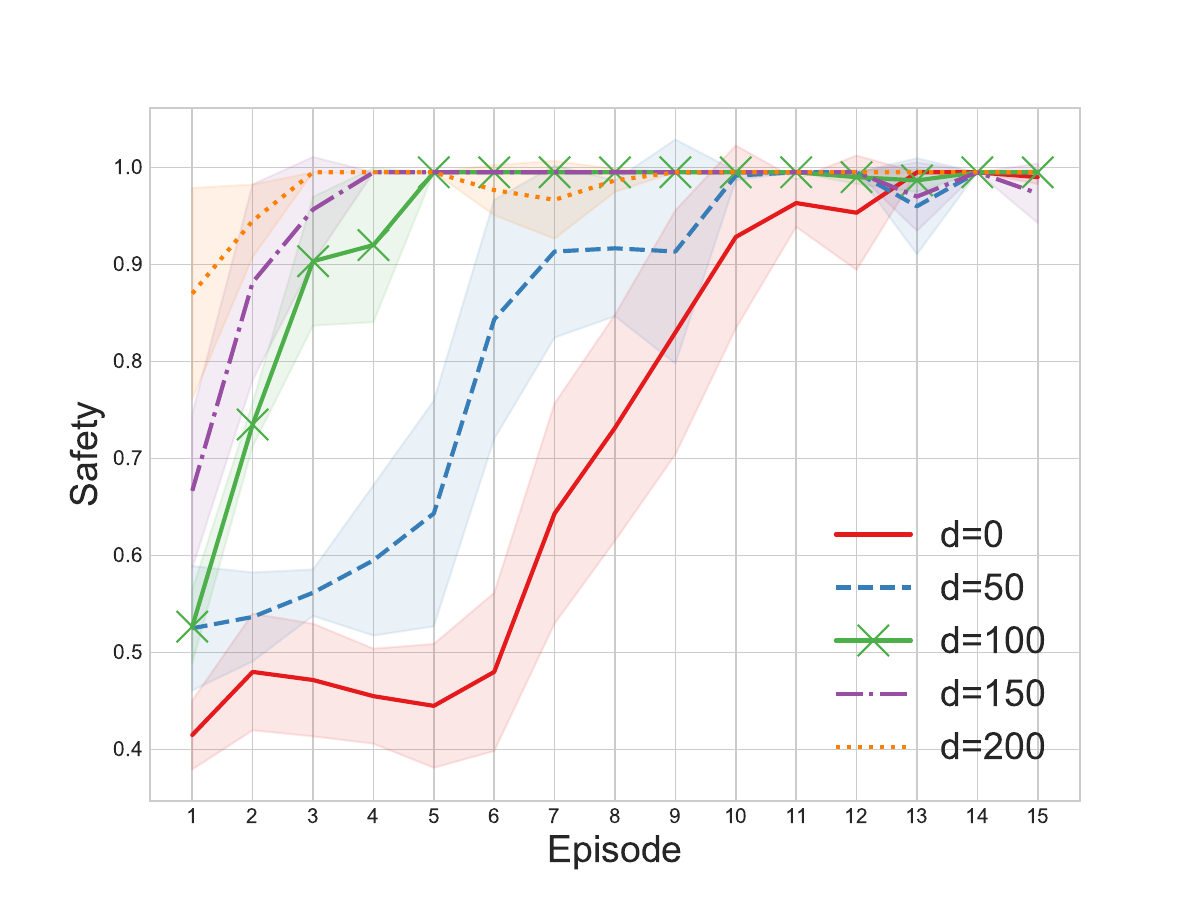}}
	\caption{Performance comparison under different communication range $d$.}
 \vspace{-15pt}
\label{fig_d}
\end{figure}

\begin{table}[t]
    	\centering
            \caption{\centering{The communication overheads of different communication range $d$.}}%居中标题
        \label{tb: communications overhead}
	    \begin{tabular}{m{2cm}|ccccc}
            \toprule
            \textbf{Name} & \multicolumn{5}{c}{\textbf{Value}}\\
	    	\midrule
	    	Communication Distance $ d $ &$0$& $50$& $100$& $150$& $200$\\
	    	\hline 
	    	Communication Overheads &$0$& $330.13$& $1,866$& $6,139.7$& $10,961.06$\\
	    	\bottomrule
	    \end{tabular}
    \end{table}	
    
\begin{figure}[!t]
	\centering
	\includegraphics[width=0.8\linewidth]{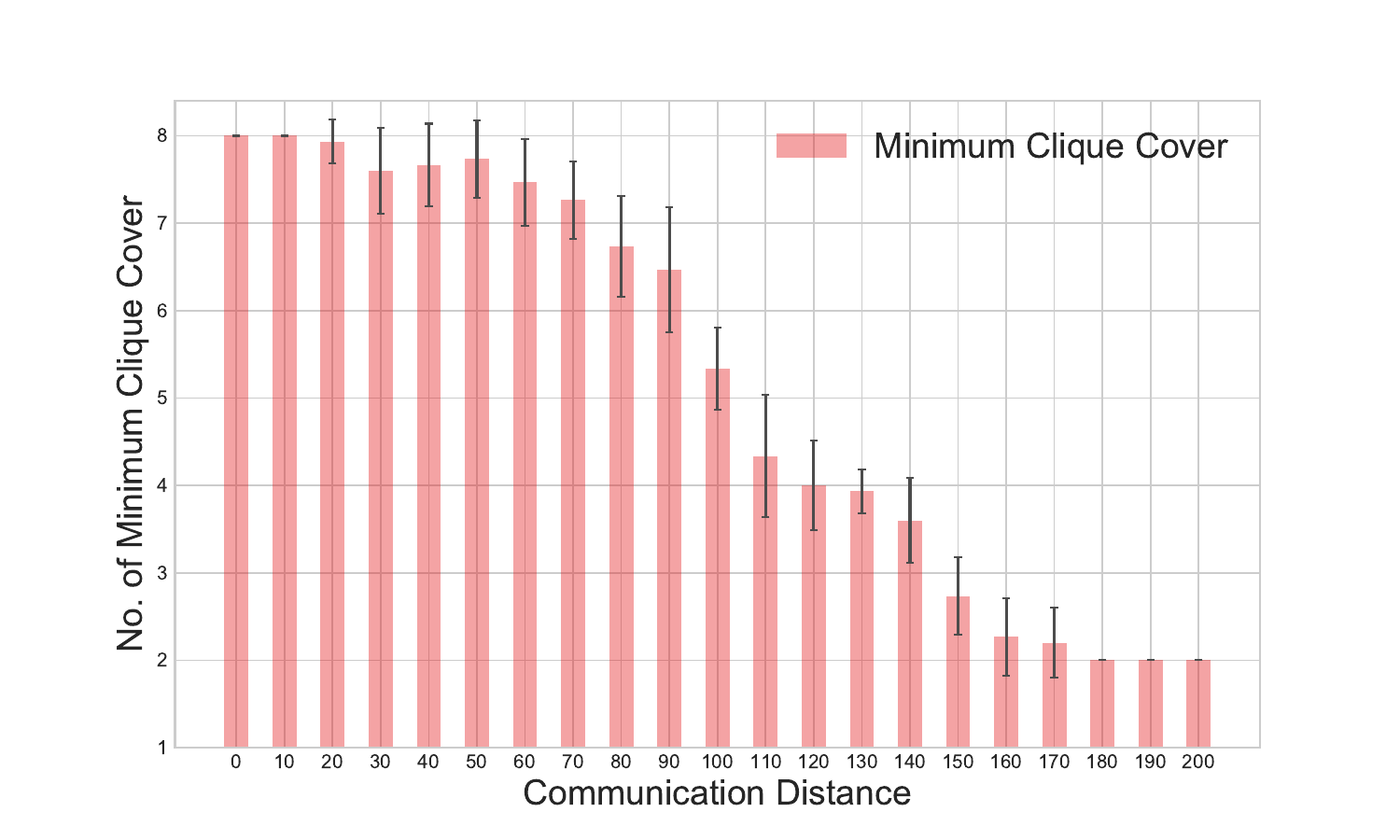}
	\caption{Correlation between number of minimum clique cover $\bar{\chi}\left(\mathcal{G}_{d,k}\right)$ and communication distance $d$.}
	\label{clique_cover}
        \vspace{-5pt}
\end{figure}
To further illustrate the feasibility and robustness of \texttt{MA-PETS} under more complex and variable real-world communication scenarios, such as vehicles experiencing communication hindrances while passing through tunnels or other challenging environments, we conduct simulations under dynamic and unpredictable communication constraints. In our simulations, we assume the possibilities of each CAV controlled by the \texttt{MA-PETS} to potentially encounter communication blockages or failures as $0\%$, $25\%$, and $50\%$ when attempting to communicate with other vehicles within a communication range of $d=100$ at the end of each episode. The results of these simulations are presented in Fig. \ref{com_constraint}. From Fig. \ref{com_constraint}, it is evident that dynamic and unpredictable communication constraints at varying probabilities can cause fluctuations in the \texttt{MA-PETS}'s safety and agility to some extent. However, the algorithm can effectively learn within $15$ episodes, demonstrating its feasibility and robustness in real-world autonomous vehicle control environments.
\begin{figure}
    \centering  %图片全局居中
         \subfigcapskip=-5pt %设置子图与子标题之间的距离
        % \subfigbottomskip=-5pt %设置子图之间的距离
    \subfigure[Agility]{
        \includegraphics[width=0.48\linewidth]{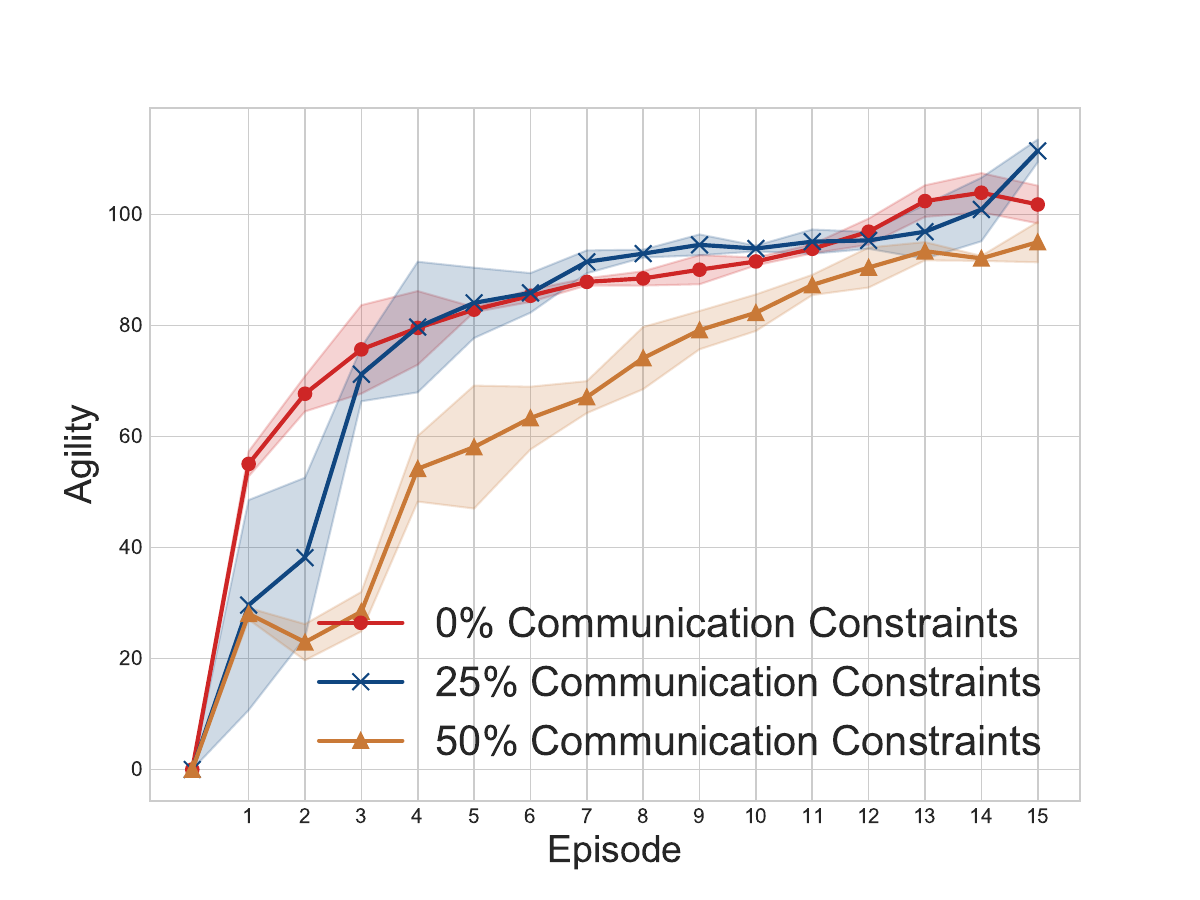}}
    \subfigure[Safety]{
        \includegraphics[width=0.48\linewidth]{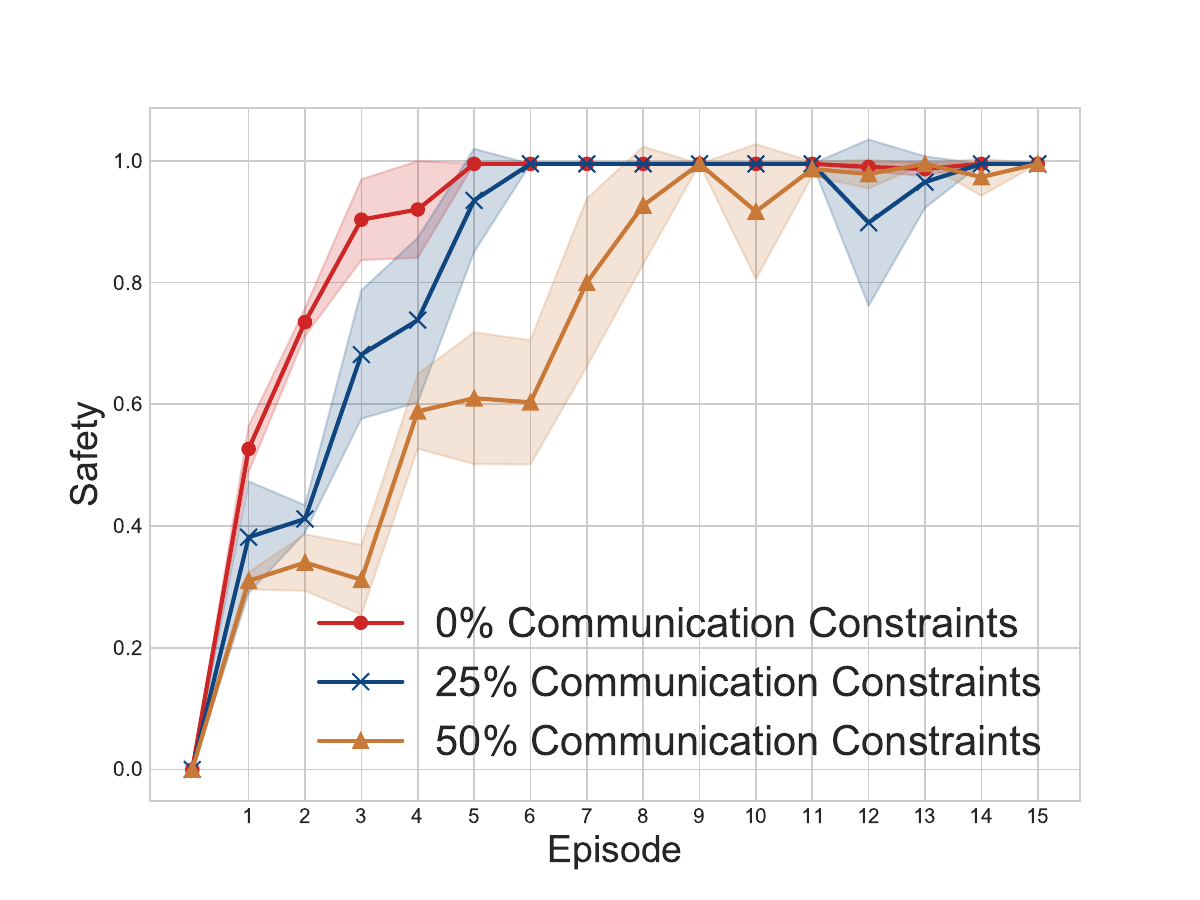}}
    \caption{Comparison of utility for different communication constraints.}
         \vspace{-5pt}
    \label{com_constraint}
\end{figure}

To clarify the effect of the MPC horizon $w$ in \texttt{MA-PETS}, we perform supplementary experiments. As depicted in Fig. \ref{fig_w}, from a security point of view, extending the MPC horizon $w$ gradually enhances security. However, beyond a horizon length of $w=25$, further increases yield only marginal security improvements. In terms of agility, the choice of the MPC horizon $w$ is crucial for the performance of the system. In particular, a ``too short'' horizon introduces a substantial bias in the performance of \texttt{MA-PETS}, as it hampers the ability to make accurate long-term predictions due to the scarcity of time steps.
Conversely, a ``too long'' horizon also results in a noticeable bias.  This is due to the increased divergence of particles over extended periods, which reduces the correlation between the currently chosen action and the long-term expected reward. This bias escalates as the horizon $w$ increases. Hence choosing an excessively long horizon adversely affects performance.

\begin{figure}
	\centering  %图片全局居中
         \subfigcapskip=-5pt %设置子图与子标题之间的距离
        % \subfigbottomskip=-5pt %设置子图之间的距离
	\subfigure[Agility]{
		\includegraphics[width=0.48\linewidth]{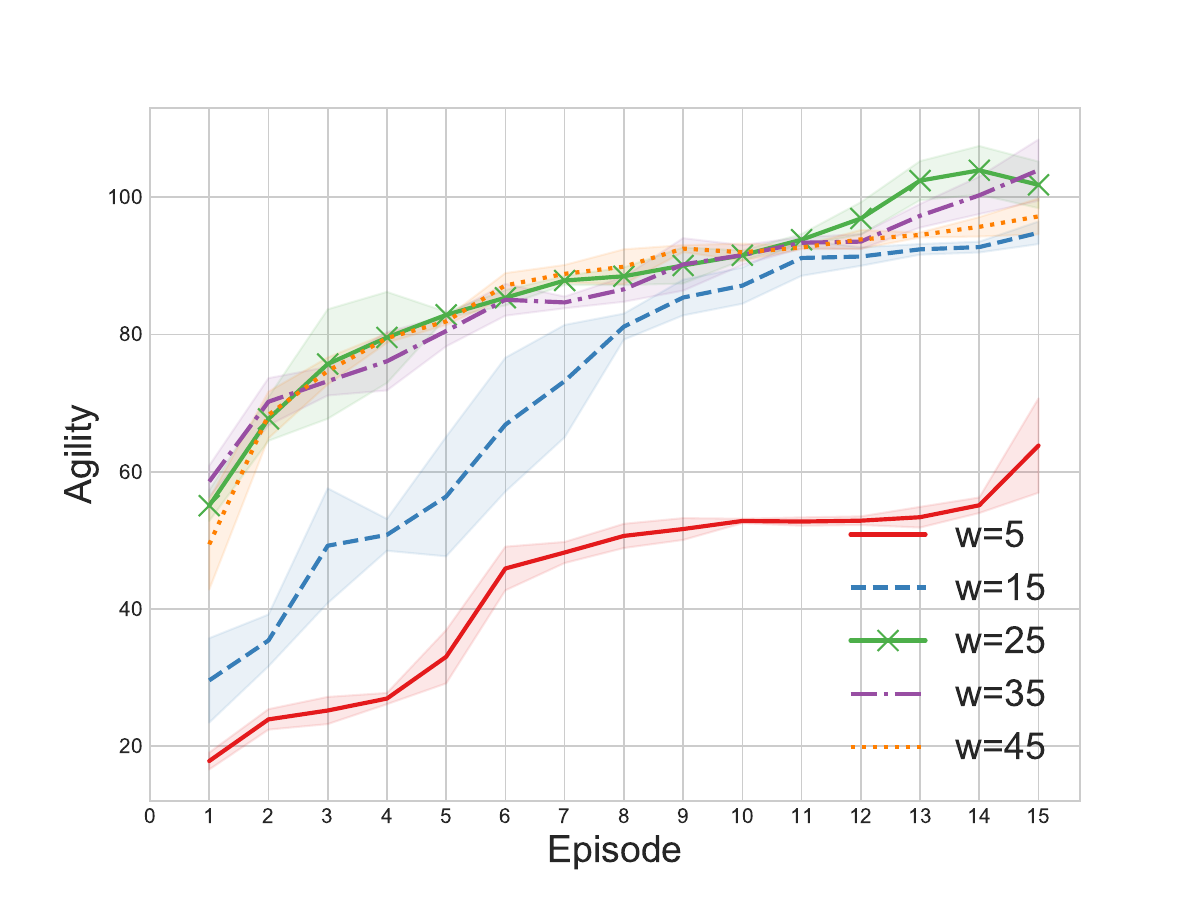}}
	\subfigure[Safety]{
		\includegraphics[width=0.48\linewidth]{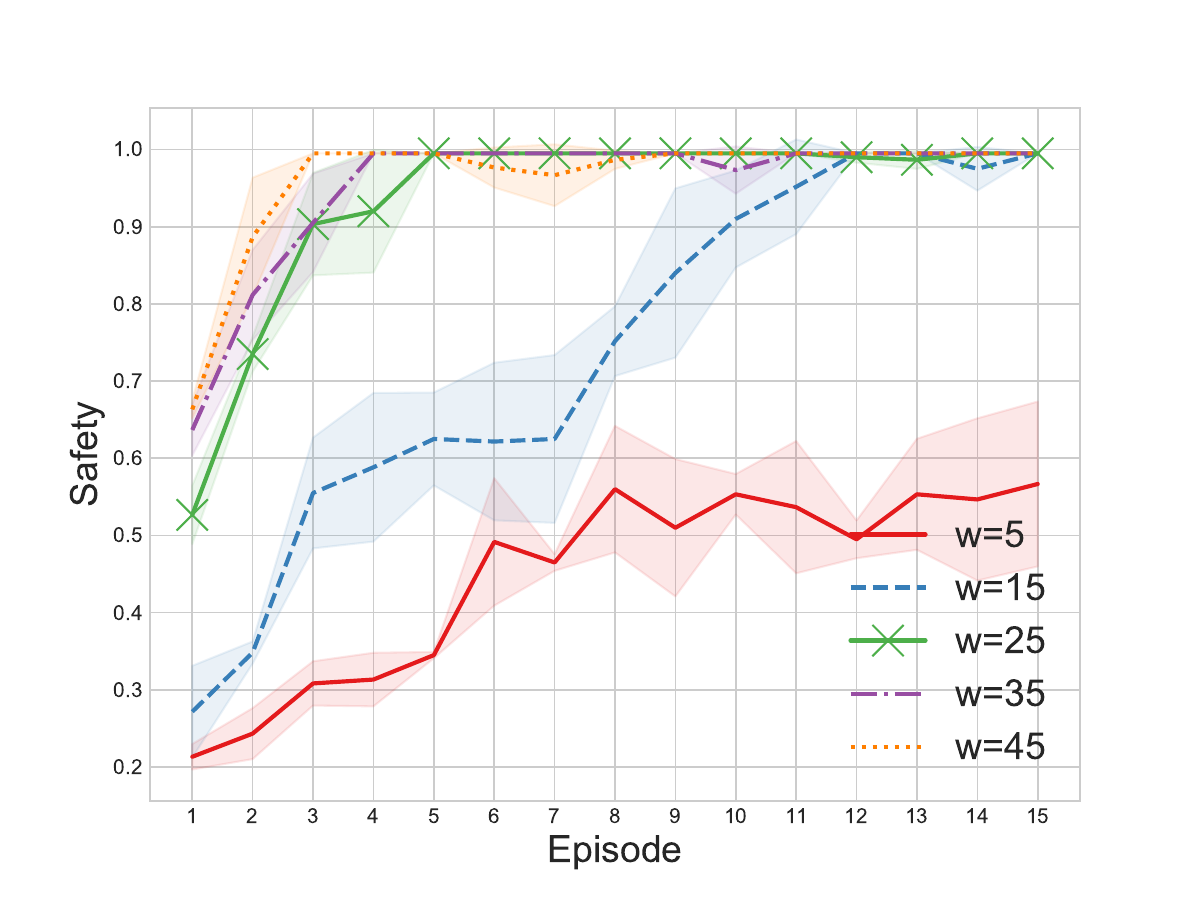}}
	\caption{Performance comparison under different horizon $w$ of MPC.}
\label{fig_w}
\vspace{-15pt}
\end{figure}

Our research also investigates the influence of the number of ensembles $B$ and the number of particles $P$ in \texttt{MA-PETS}. In terms of agility, Fig. \ref{box_ensembles} and Fig. \ref{box_particles} show the corresponding results respectively, which are derived from $15$ training episodes through $15$ independent simulation runs. It can be observed from Fig. \ref{box_ensembles} that along with the increase in $B$, the learning becomes more regularized and the performance improves. However, the performance improvement is no longer apparent for a sufficiently large $B$, and this improvement is more pronounced in more challenging and complex environments that require the learning of intricate dynamical models, leaving more scope for effective exploitation of the strategy without the use of model integration. In Fig. \ref{box_particles}, superior performance can be reaped for a larger number of particles, because more particles allow for a more accurate estimation of the reward for state-action trajectories influenced by transition probabilities. 
\begin{figure}%[!t]
	\centering
	\includegraphics[width=0.6\linewidth]{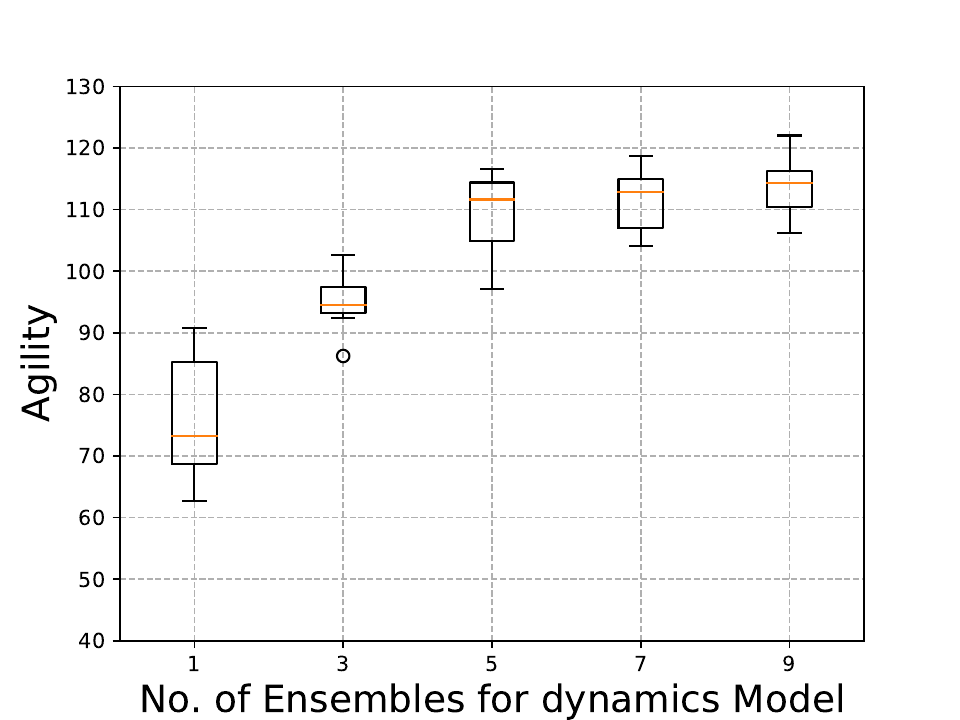}
	\caption{Performance comparison under different ensembles for dynamics model.}
    % \vspace{-15pt}
	\label{box_ensembles}
\end{figure}
\begin{figure}%[!t]
	\centering
	\includegraphics[width=0.6\linewidth]{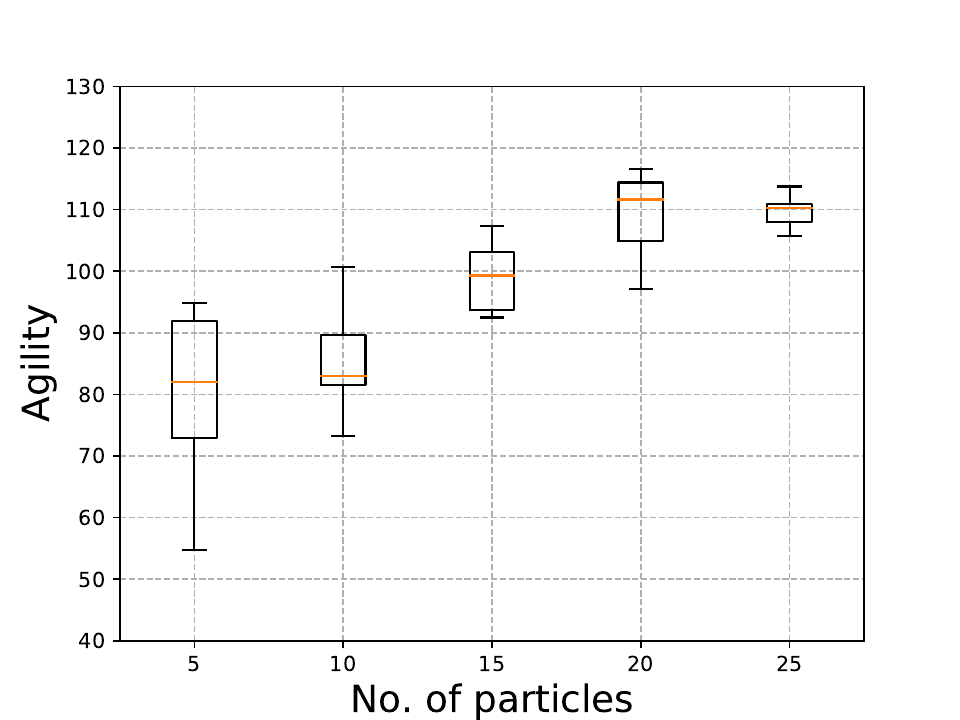}
	\caption{Performance comparison under different particles for TS.}
	\label{box_particles}
 \vspace{-5pt}
\end{figure}

To conduct a more in-depth numerical analysis of the complexity and computational time of \texttt{MA-PETS}, we compare the average decision-making time and its standard deviation of different MARL methods after learning 15 episodes, in both single-lane loop and multi-lane loop scenarios as shown in Table \ref{tb: decision_making_time}. The results reveal that the average decision time of \texttt{MA-PETS} is slightly higher than that of \texttt{FIRL} and \texttt{MAPPO} in both scenarios, but faster than \texttt{SVMIX}. Overall, compared to \texttt{FIRL} and \texttt{MAPPO}, \texttt{MA-PETS} has smaller variance and lower variability in decision times, demonstrating greater stability.
\begin{table*}
     \caption{\centering{The Comparison of Means and Standard Deviations of the Average Decision-making Time per Step in Testing MARL Methods in ``Figure Eight". (Unit: \num{1e-2} Seconds)}}
    \label{tb: decision_making_time}
    % \vspace{-0.2cm}
    % \tiny
    % \scriptsize
    \footnotesize
    \begin{center}
      \renewcommand\arraystretch{1.5} {
            \begin{tabular}{|c|c|c|c|c|}
                \hline	\multirow{2}{*} {\textbf{Scenarios}} & \multicolumn{4}{c|}{\textbf{The Different MARL Methods}}       \\
                \cline{2-5} & \texttt{MA-PETS} & \texttt{FIRL}& \texttt{MAPPO} & \texttt{SVMIX}\\
               % \midrule \textbf{Single-lane loop} & $ \num{8.851e-2}\pm\num{3.105e-3}$ & $ \num{7.604e-2}\pm\num{3.034e-3}$ & $ \num{8.007e-2}\pm\num{5.486e-3}$ & $ \num{1.028e-1}\pm\num{2.668e-3}$ \\
                %\hline \textbf{Multi-lane loop} & $ \num{9.726e-2}\pm\num{2.615e-3}$ & $ \num{8.104e-2}\pm\num{4.519e-3}$ & $ \num{8.779e-2}\pm\num{4.607e-3}$  & $\num{1.106e-1}\pm\num{3.025e-3}$    \\
                \hline \textbf{Single-lane loop} & $ 8.851 \pm 0.3105 $ & $ \num{7.604}\pm\num{0.3034}$ & $ \num{8.007}\pm\num{0.5486}$ & $ \num{10.28}\pm\num{0.2668}$ \\
                \hline \textbf{Multi-lane loop} & $ \num{9.726}\pm\num{0.2615}$ & $ \num{8.104}\pm\num{0.4519}$ & $ \num{8.779}\pm\num{0.4607}$  & $\num{11.06}\pm\num{0.3025}$    \\
                \hline
            \end{tabular}
        }
    \end{center}
    \vspace{-0.6cm}
\end{table*}
\section{Conclusion and Discussion}
\label{sec: conclusion}
In this paper, we have studied a decentralized MBRL-based control solution for CAVs with a limited communication range. In particular, we have proposed the \texttt{MA-PETS} algorithm with a significant performance improvement in terms of sample efficiency. Specifically, \texttt{MA-PETS} learns the environmental dynamics model from samples communicated between neighboring CAVs via PE-DNNs. Subsequently, \texttt{MA-PETS} efficiently develops TS-based MPC for decision-making. Afterward, we have derived \texttt{UCRL2}-based group regret bounds, which theoretically manifests that in the worst case, limited-range communications in multiple agents still benefit the learning. We have validated the superiority of \texttt{MA-PETS} over classical MFRL algorithms and demonstrated the contribution of communication to multi-agent MBRL.

There are many interesting directions to be explored in the future. For example, \texttt{MA-PETS} encounters significant challenges in more realistic scenarios. As environmental complexity escalates, scenarios with numerous variables or agents heighten the learning and computational requirements of \texttt{MA-PETS}. Moreover, if the dynamics of the environment are unpredictable, \texttt{MA-PETS} may face difficulties due to the probabilistic nature of its ensemble predictions. Another critical factor is the availability of computational resources. The computational power at hand can significantly influence the performance of \texttt{MA-PETS}, especially as more intricate environments or accurate ensemble methods demand robust processing capabilities. The key next steps involve optimizing the algorithm, reducing its reliance on computational resources, minimizing its decision delays, and improving its adaptability to rapidly changing and complex environments. Furthermore, \texttt{MA-PETS} confronts significant OOD challenges resulting from scarce training data, which can lead to learning instability and substantial overhead. Consequently, designing more accurate models to mitigate OOD issues is a critical priority that requires urgent attention. Finally, we have only derived the worst case of group regret bound, and the results for more general cases can be explored.

\appendix
\begin{lemma}[Appendix C.3 of \cite{UCRL2}]
\label{lem: inequality}
For any sequence of numbers $x_1, \ldots, x_n$ with $0 \leq x_k \leq X_{k-1}:=\max \left\{1, \sum_{i=1}^{k-1} x_i\right\}$, we have
\begin{align}
 \sum\nolimits_{k=1}^n \frac{x_k}{\sqrt{X_{k-1}}} \leq(\sqrt{2}+1) \sqrt{X_n}.   
\end{align}
\end{lemma}
% 参考文献中期刊名称要么全部全称、要么全部缩略词
% 快速看了一下，部分会议论文如NIPS或者ICML也没有统一，如20、33、42
% 一些会议缺少会议地点
\bibliographystyle{IEEEtran}
\bibliography{bib1}

\vspace{-1.5cm}
\begin{IEEEbiography}[{\includegraphics[width=1in,height=1.25in,clip,keepaspectratio]{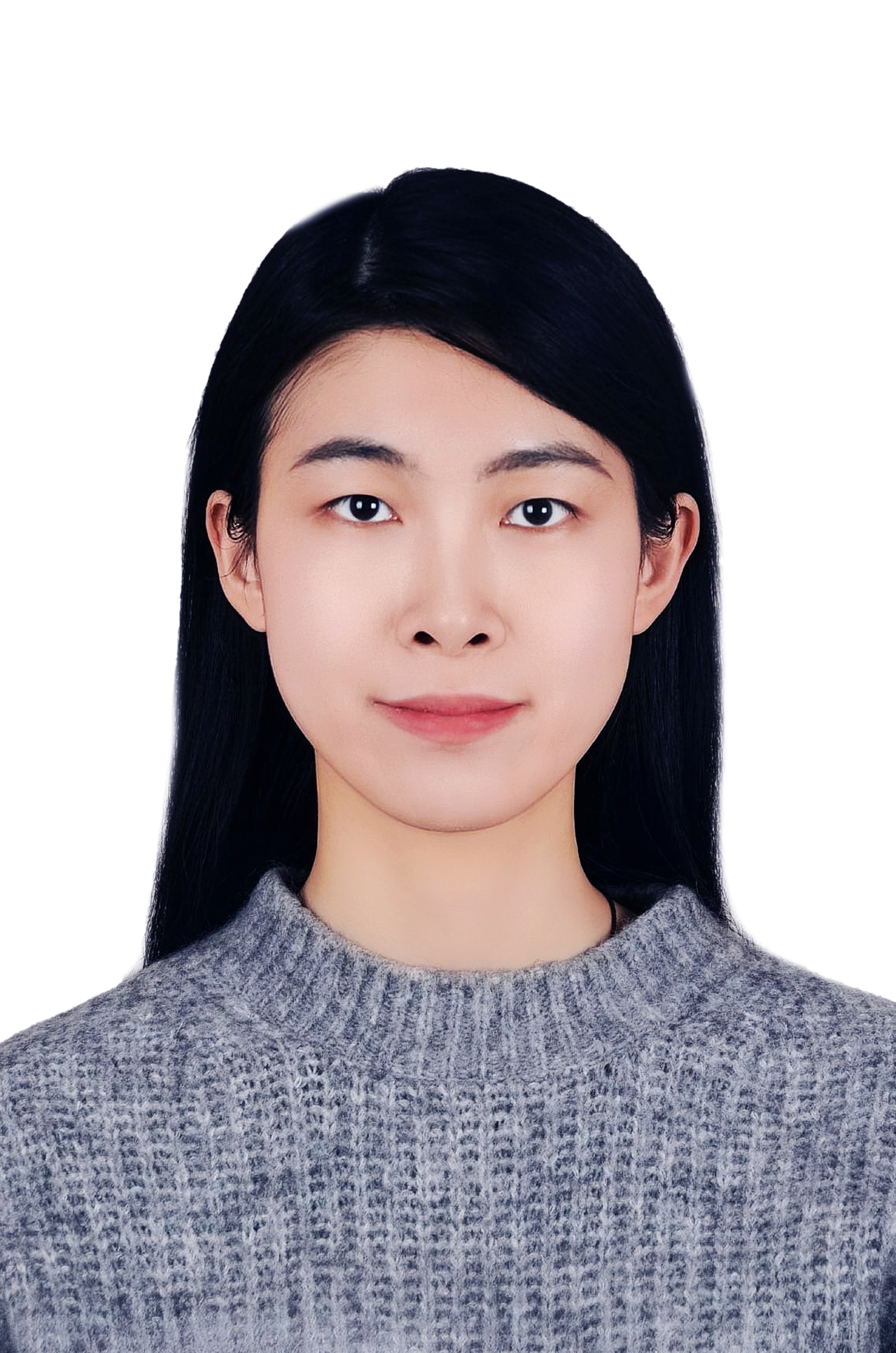}}]{Ruoqi Wen}
received the B.Sc. degree in Mathematics and Applied Mathematics from Hefei University of Technology, Hefei, China, in June 2021. She is currently pursuing a Ph.D. degree with the College of Information Science and Electronic Engineering at Zhejiang University,
Hangzhou, China. Her research interests include multi-agent reinforcement learning, regret bound, endogenous and exogenous uncertainty analysis, wireless networks, and autonomous vehicle control.
\end{IEEEbiography}

\vspace{-1cm} 
\begin{IEEEbiography}
[{\includegraphics[width=1in,height=1.25in, clip,keepaspectratio]{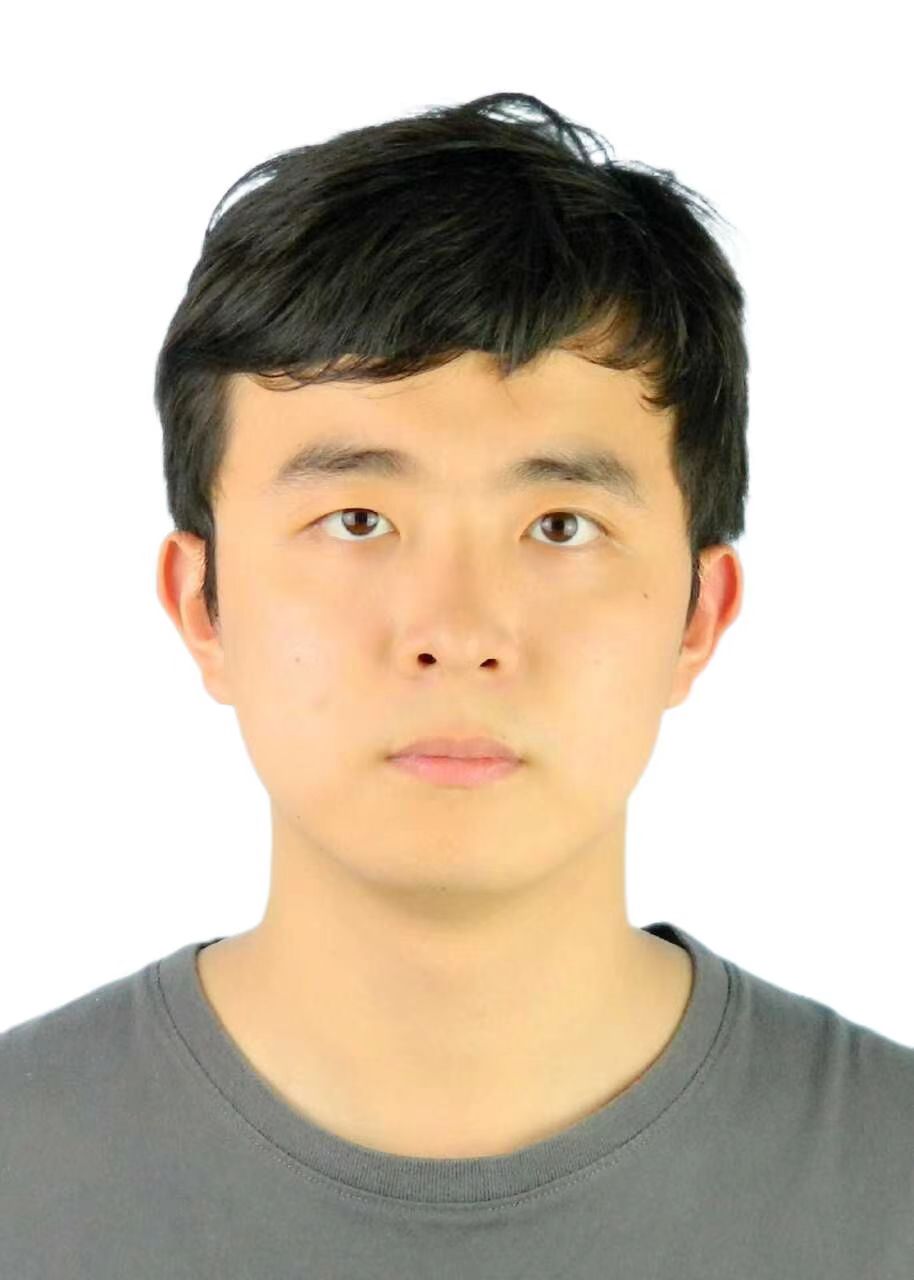}}]{Jiahao Huang} 
received the B.E. degree in information engineering from Zhejiang University, Hangzhou, China, in 2023. He is currently working toward a Master's degree in the Department of Information Science and Electronic Engineering at Zhejiang University, Hangzhou, China. His research interests include machine learning, signal detection, vehicular communications, and radio resource management.
\end{IEEEbiography}

\vspace{-1cm} 
\begin{IEEEbiography}[{\includegraphics[width=1in,height=1.25in,clip,keepaspectratio]{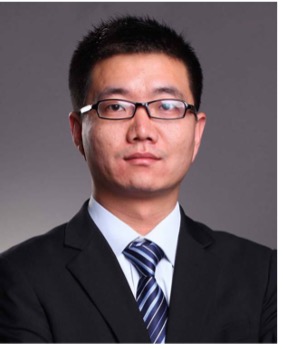}}]{Rongpeng Li}
(Member, IEEE) is currently an Associate Professor with the College of Information Science and Electronic Engineering, Zhejiang University, Hangzhou, China. He was a Research Engineer with the Wireless Communication Laboratory, Huawei Technologies Company, Ltd., Shanghai, China, from August 2015 to September 2016. He was a Visiting Scholar with the Department of Computer Science and Technology, University of Cambridge, Cambridge, U.K., from February 2020 to August 2020. His research interest currently focuses on networked intelligence for communications evolving (NICE). He received the Wu Wenjun Artificial Intelligence Excellent Youth Award in 2021. He serves as an Editor for China Communications.
\end{IEEEbiography}

\vspace{-1cm} 
\begin{IEEEbiography}[{\includegraphics[width=1in,height=1.25in,clip,keepaspectratio]{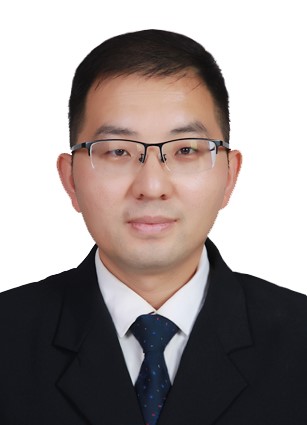}}]{Guoru Ding}
is currently a Professor at the College of Communications Engineering, Nanjing, China. He received a B.S. (Hons.) degree in electrical engineering from Xidian University, Xi’an, China, in 2008, and a Ph.D. (Hons.) degree in communications and information systems from the College of Communications Engineering, Nanjing, China, in 2014. From 2015 to 2018, he was a Post-Doctoral Research Associate with the National Mobile Communications Research Laboratory, Southeast University, Nanjing, China. His research interests include cognitive radio networks, massive MIMO, machine learning, and data analytics over wireless networks.

He received the Excellent Doctoral Thesis Award of the China Institute of Communications in 2016, the Alexander von Humboldt Fellowship in 2017, and the 14th IEEE COMSOC Aisa-Pacific Outstanding Young Researcher Award in 2019. He was a recipient of the Natural Science Foundation for Distinguished Young Scholars of Jiangsu Province, China, and six best paper awards from international conferences such as the IEEE VTC-FALL 2014. He has served as a Guest Editor for the IEEE JOURNAL ON SELECTED AREAS IN COMMUNICATIONS (special issue on spectrum sharing and aggregation in future wireless networks) and the Chinese Journal of Aeronautics (special issue on when aeronautics meets 6G and AI). 
\end{IEEEbiography}
\vspace{-1cm}

\begin{IEEEbiography}[{\includegraphics[width=1in,height=1.25in,clip,keepaspectratio]{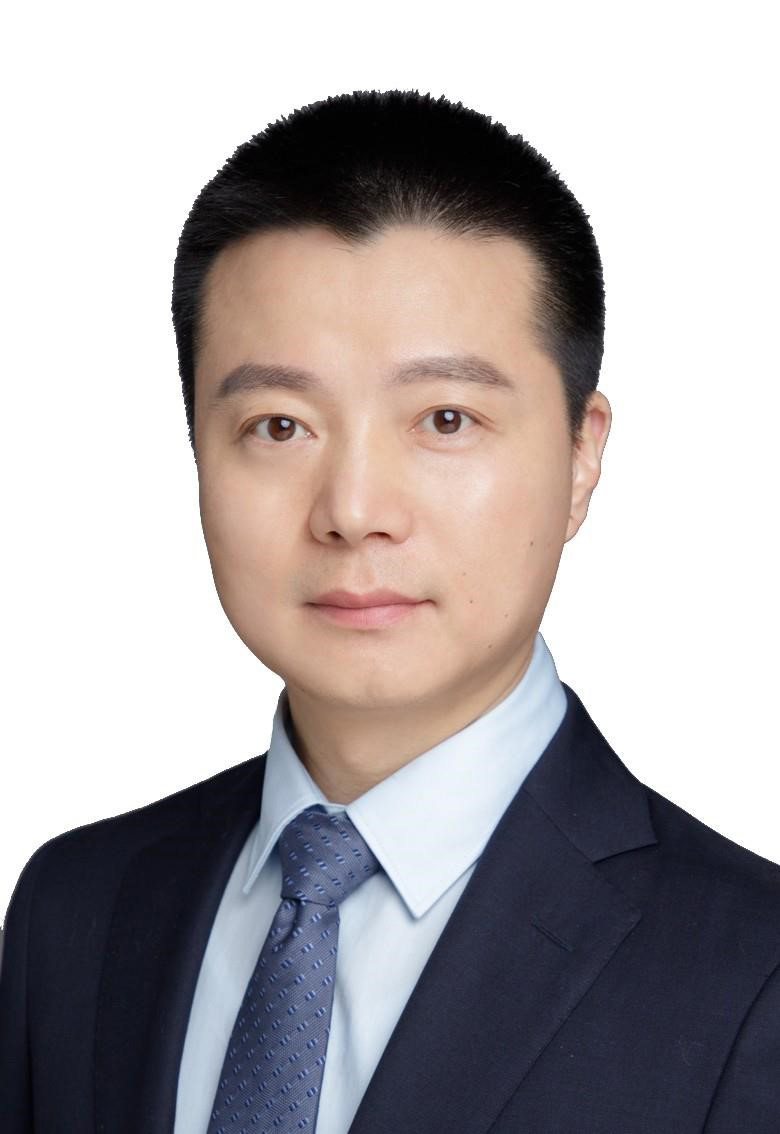}}]{Zhifeng Zhao}
(Member, IEEE) received the B.E. degree in computer science, the M.E. degree in communication and information systems, and the Ph.D. degree in communication and information systems from the PLA University of Science and Technology, Nanjing, China, in 1996, 1999, and 2002, respectively. From 2002 to 2004, he acted as a Post-Doctoral Researcher at Zhejiang University, Hangzhou, China, where his researches were focused on multimedia next-generation networks (NGNs) and soft switch technology for energy efficiency. From 2005 to 2006, he acted as a Senior Researcher with the PLA University of Science and Technology, where he performed research and development on advanced energy-efficient wireless routers, ad-hoc network simulators, and cognitive mesh networking test-bed. From 2006 to 2019, he was an Associate Professor at the College of Information Science and Electronic Engineering, Zhejiang University. Currently, he is with the Zhejiang Lab, Hangzhou as the Chief Engineering Officer. His research areas include software-defined networks (SDNs), wireless networks in 6G, computing networks, and collective intelligence. He is the Symposium Co-Chair of ChinaCom 2009 and 2010. He is the Technical Program Committee (TPC) Co-Chair of the 10th IEEE International Symposium on Communication and Information Technology (ISCIT 2010).
\end{IEEEbiography}

\end{document}